\makeatletter
\disable@package@load{breakurl}{}
\makeatother

\documentclass[sn-mathphys-num]{sn-jnl}

\usepackage{graphicx}%
\usepackage{multirow}%
\usepackage{amsmath,amssymb,amsfonts}%
\usepackage{amsthm}%
\usepackage{rsfso}
\usepackage[title]{appendix}%
\usepackage{xcolor}%
\usepackage{textcomp}%
\usepackage{manyfoot}%
\usepackage{booktabs}%
\usepackage{soul}%
\usepackage{natbib}

\usepackage{listings}%
\usepackage{subcaption}
\usepackage{makecell}
\usepackage[section]{placeins}
\usepackage{needspace}
\usepackage{graphicx,subcaption,ragged2e}
\usepackage{geometry}

\usepackage[ruled,longend]{algorithm2e}
\usepackage{textgreek}
\usepackage{amssymb}
\usepackage{lmodern}
\usepackage{cancel}
\usepackage[switch]{lineno}

\renewcommand\appendix{\par
    \setcounter{section}{0}
    \setcounter{subsection}{0}
    \gdef\thesection{\Alph{section}}}

\makeatletter
\def\BState{\State\hskip-\ALG@thistlm}
\makeatother

\newgeometry{vmargin={15mm}, hmargin={17mm,17mm}}   

\usepackage{changes}

\newtheorem{theorem}{Theorem}
\newtheorem{lemma}{Lemma}
\author[1]{\sur{Ji} \fnm{Chen}}

\author[6]{\sur{Song} \fnm{Chen}$^*$}

\author[1]{\sur{Chengzhang} \fnm{Gong}}
\author[3,1]{\sur{Li} \fnm{Fan}$^*$}

\author[1,2,3,4,5]{\sur{Chao} \fnm{Xu} $^*$}

\affil[1]{\orgdiv{College of Control Science and Engineering}, \orgname{Zhejiang University}, \orgaddress{\city{Hangzhou}, \postcode{310000}, \state{Zhejiang}, \country{China}}}

\affil[2]{\orgdiv{Institute of Cyber-Systems and Control}, \orgname{Zhejiang University}, \orgaddress{\city{Hangzhou}, \postcode{310000}, \state{Zhejiang}, \country{China}}}

\affil[3]{\orgdiv{Huzhou Institute of Zhejiang University}, \orgaddress{\city{Huzhou}, \postcode{313000}, \state{Zhejiang}, \country{China}}}

\affil[4]{\orgdiv{State Key Laboratory of Industrial Control Technology}, \orgname{Zhejiang University}, \orgaddress{\city{Hangzhou}, \postcode{310000}, \state{Zhejiang}}, \country{China}}

\affil[5]{\orgdiv{Zhejiang Provincial Engineering Research Center for Intelligent Mobile Unmanned Systems Technology and Huzhou Key Lab for Autonomous Systems}, \orgname{Huzhou Institute of Zhejiang University}, \orgaddress{\city{Huzhou}, \postcode{313000}, \state{Zhejiang}, \country{China}}}

\affil[6]{\orgdiv{Department of Mathematics}, \orgname{National University of Singapore}, \orgaddress{\city{10 Lower Kent Ridge Road}, \postcode{119076}, \country{Singapore}}}

\begingroup

\footnotetext[1]{Corresponding authors: Chao Xu (cxu@zju.edu.cn), Li Fan (fanli77@zju.edu.cn), Song Chen (song.chen@nus.edu.sg)}
\endgroup

\begin{document}

\title[Article Title]{Swarm-Inspired Generation of Collective Behaviors in Graph Dynamical Systems}


\abstract{
    Collective behavior arises when locally interacting units produce coordinated global organization, from synchronization in dynamical systems to task-relevant information flow on graphs. The central challenge is not only to explain how collective behavior emerges, but to design local interaction rules that can produce desired global organization and generalize across graphs, dynamics and tasks.  
    To address this challenge, we introduce the Swarm-Inspired Emergent Synchronizer (SIES), a graph-dynamical framework that learns generalizable local-interaction laws for controllable collective organization. Each node is an agent-like dynamical unit with a state and task cue, and signed source-target-conditioned attention acts as an adaptive coupling term inside an explicit evolution model. Therefore, SIES combines an explicit dynamical engine with local agent intelligence, similar to biological swarms. For synchronization control, SIES learns a generalizable coupling operator that produces prescribed synchronization patterns for CDSs across untrained network scales, target phase relations, and intrinsic node dynamics without retraining. The learned operator also reaches gait-related modes faster than three oscillator baselines and generalizes synchronization-driven locomotion to simulated multi-legged robots of different scales and a physical hexapod after leg disablement. For graph representation learning, SIES applies the same signed interaction principle to message passing and achieves the highest performance among the compared methods on heterophilous node-classification benchmarks. Together, these results position SIES as a generalizable and learnable graph-dynamical interaction framework with promise for synchronization control, adaptive robot coordination, and heterophilous graph representation learning.
}




\maketitle

\section{Introduction}\label{sec1}    
\begin{figure}[t!]  
	\centering  
	\includegraphics[width=\textwidth]{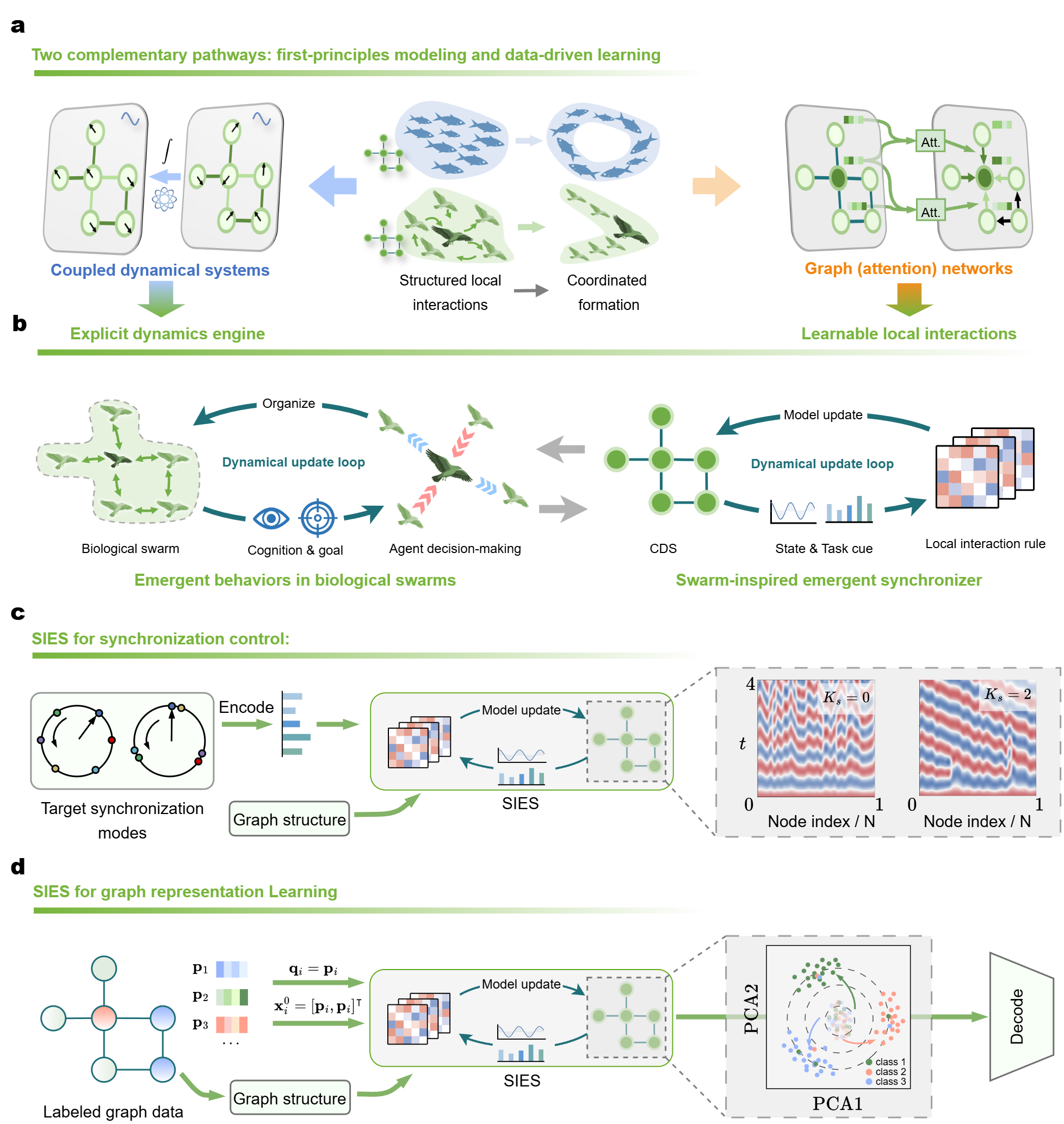}  
    \caption{\textbf{Overview of collective-behavior modeling and SIES.} \textbf{(a)} Coupled dynamical systems and graph neural networks provide two complementary routes for modeling collective behavior on graph-structured systems. \textbf{(b)} SIES combines an explicit dynamics engine with learnable local interaction rules, thereby translating the dynamic update loop of biological swarms into a computational framework. \textbf{(c)} SIES for synchronization control: target synchronization modes and graph structure are encoded as inputs to SIES, whose learned coupling rule drives the CDS toward prescribed collective rhythms. \textbf{(d)} SIES for graph representation learning: labeled graph data provide node features and graph structure, SIES evolves node representations through dynamical interactions, and the resulting representation trajectory is decoded for downstream prediction.}
	\label{fig:intro}  
\end{figure}  

Collective behavior emerges when many simple units interact locally and organize into coherent group-level patterns~\cite{bonabeau1999swarm,Vicsek1995NovelTO,Strogatz1993CoupledOA}. In nature, birds align with nearby flockmates to form directional flight formations, fish adjust to neighboring motion to create schooling or circling patterns, and fireflies synchronize their flashes through local signals (Fig.~\ref{fig:intro}a)~\cite{rahmani2020flocking,heras2019deep,xiao2024perception,Buck1966BiologyOS}. Similar principles appear in excitable tissues, power grids and relational data, where local interactions can produce traveling waves, synchronized operation or task-relevant information flow~\cite{Muller2018CorticalTW,Davidenko1992StationaryAD,Filatrella2007AnalysisOA,Drfler2012SynchronizationIC,velivckovic2017graph,Platonov2023ACL}. These examples raise a central question: how do graph-local interaction rules give rise to global collective behaviors, and how can such rules be designed for controllable tasks?

To answer this question, researchers model and exploit these collective behaviors under graph-structured local interaction rules via the frameworks of coupled dynamical systems (CDSs) and graph neural networks (GNNs), which provide two complementary routes to this question (Fig.~\ref{fig:intro}a). CDSs provide a mathematically grounded, first-principles route to this problem by describing how local coupling among dynamical nodes generates, stabilizes and controls collective states. This route captures collective phenomena such as coordinated animal motion~\cite{Vicsek1995NovelTO}, synchronous flashing in fireflies~\cite{Buck1966BiologyOS}, traveling neural activity~\cite{Muller2018CorticalTW} and waves in excitable tissue~\cite{Davidenko1992StationaryAD}, and underpins engineered systems ranging from power-grid synchronization~\cite{Filatrella2007AnalysisOA,Drfler2009SynchronizationAT,Drfler2012SynchronizationIC,Sajadi2021SynchronizationIE} to central pattern generators (CPGs) for robotic locomotion~\cite{Grillner1975LocomotionIV,ijspeert2014biorobotics,ijspeert2008central,radosavovic2024real,chen2021rhythm}. GNNs, by contrast, provide a data-driven route for computational collective information processing: node representations are repeatedly updated through local message passing so that relational structure can produce task-relevant graph representations, including for node classification~\cite{velivckovic2017graph,Platonov2023ACL,Luo2024ClassicGA}. 
Both frameworks rely on the same fundamental substrate: structured local interactions that determine how neighboring units influence one another (Fig.~\ref{fig:intro}a). Moreover, a sharper understanding of their respective strengths and limitations comes from comparing them to the broader framework of swarm intelligence~\cite{bonabeau1999swarm,Vicsek1995NovelTO}. As shown in Fig.~\ref{fig:intro}b, in a biological swarm, collective organization arises from systems that simultaneously maintain an explicit dynamics engine (state evolution governed by interaction laws) and agent intelligence (local, adaptive decision-making that optimizes toward task or collective objectives). Swarm units continuously evaluate neighbors and adjust their influence in a goal-directed, decentralized manner.

From this vantage point, classical CDS constructions primarily supply the explicit dynamics engine and deliver mechanistic clarity through stability analysis and interpretable coupling functions. Nevertheless, because their rules are typically prescribed or symmetry-derived for specific patterns and scales~\cite{collins1993coupled,Buono2001ModelsOC,stewart2003symmetry,golubitsky2005patterns}, they are highly sensitive to initial conditions \cite{lucas2019synchronisation,wilkins2009sensitivity,chen2025free}, lack mechanisms for self-improvement or data-driven optimization of dynamical characteristics, and are rarely optimized end-to-end for reachability and convergence. In contrast, these capabilities are routinely achieved by swarm agents through local adaptation. GNNs, conversely, excel at the agent intelligence aspect by learning flexible local aggregation rules from data. They typically lack an explicit dynamics engine, so message passing simply mixes features in a statistical manner and carries no intrinsic dynamical semantics. Consequently, GNNs struggle to capture heterophilous repulsion, which drives neighboring nodes to actively oppose one another. Repulsion plays a fundamental role in swarm systems by maintaining diversity and enabling complex collective patterns. Standard GNNs thus develop a strong homophily bias and over‑smoothing tendency~\cite{Chamberlain2021GRANDGN,Rusch2022GraphCoupledON,Nguyen2023FromCO}, and they fail to support rich, evolutionarily meaningful collective modes.

We therefore hypothesize that an effective framework must simultaneously possess both a complete explicit dynamics engine and agent intelligence. In such a framework, local interaction rules are learned and optimized as agent intelligence and executed as driving terms inside an explicit state-evolution law provided by the dynamics engine. Such a framework would support a richer variety of collective modes through flexible, signed and task-conditioned couplings while retaining strong dynamical properties such as stability, convergence behavior and cross-scale generalization. Such an integration could bridge first-principles dynamical modeling and data-driven learning
~\cite{Carleo2019MLPhysical,Karniadakis2021PhysicsInformed,Wang2023AIDiscovery,chen2024accelerated}. This framework advances both frontiers simultaneously: it endows first-principles CDS with data-informed adaptability and scalability, while cross-fertilizing GNN-driven graph intelligence with explicit physical semantics to naturally accommodate multi-modal repulsion and heterophily.

Motivated by this hypothesis, we introduce the Swarm-Inspired Emergent Synchronizer (SIES), shown in Fig.~\ref{fig:intro}b. SIES treats each dynamical node in a CDS as an agent-like unit that possesses both its intrinsic evolving state and a local task-conditioning cue (its ``swarm objective''). A learned graph interaction rule applies signed attention to determine how neighboring nodes positively or negatively influence each other, depending on the state and task cue. These influences are injected directly into the CDS as the coupling term that drives the node's explicit dynamical evolution. In this unified architecture, the local interaction rule is simultaneously data-informed and optimizable (agent intelligence) and executed through genuine state integration (explicit dynamics engine), allowing collective organization to emerge under task guidance while respecting underlying dynamical laws. More mathematical details are provided in Methods.  

We instantiate SIES in two complementary settings that differ in the origin of the swarm objective and the nature of the training signal, yet both rely on signed, graph-structured, task-conditioned coupling to generate collective dynamics. In the synchronization-control setting (Fig.~\ref{fig:intro}c), the objective encodes target collective patterns and the interaction rule is optimized via reinforcement learning on dynamical rollouts. In the graph representation-learning setting (Fig.~\ref{fig:intro}d), node features supply the objective and the coupled node representations are trained end-to-end for downstream tasks such as node classification. To determine whether SIES successfully reproduces key emergent properties of natural swarms while simultaneously enhancing the capabilities of both coupled dynamical systems and graph neural networks, we design a sequence of investigations anchored in these instantiations.

To evaluate whether SIES captures key emergent properties of natural swarms, we employ its synchronization control instantiation as a testbed and systematically examine generalization, convergence, and sparse interaction within CDSs. In the generalization experiments, a trained SIES model generalizes to unseen system scales, unseen phase targets, and unseen intrinsic node dynamics without retraining. In the convergence experiments, SIES is compared against three established oscillator baselines: a fully connected model~\cite{righetti2006design,Righetti2008PatternGW}, a nearest-neighbor Salamander model~\cite{ijspeert2007swimming}, and a diffusively coupled model~\cite{Yu2016GaitGW}. Across sampled initial conditions, SIES reaches the prescribed synchronization modes more rapidly and exhibits a linear-like dependence on the initial phase distance. Furthermore, in sparse-interaction experiments, SIES demonstrates robust pattern formation capability in sparse and irregular CDS networks. Collectively, these synchronization-control results demonstrate that SIES aligns with natural swarms, in which collective behaviors self-organize and emerge from sparse interactions irrespective of group size or individual dynamics. In SIES, these properties arise from signed, task-conditioned local couplings executed within a CDS. This swarm-aligned perspective motivates the use of SIES for embodied rhythmic control. We therefore apply SIES to locomotion in simulated centipede-like robots and on a physical hexapod platform. In these systems, the same learned SIES maps the current leg-foot topology to coordinated rhythmic signals and remains effective when the embodiment graph changes. The centipede simulations assess transfer across body lengths, while the hexapod experiment evaluates zero-shot adaptation after sequential leg disablement. This damage-tolerant, topology-aware behavior reflects the swarm principle that collective function can persist despite loss or reconfiguration of individual agents.

Finally, we evaluate SIES as a graph representation-learning mechanism by benchmarking it against five strong baselines on heterophilous node-classification datasets: enhanced GCN, GAT and GraphSAGE variants~\cite{Kipf2016SemiSupervisedCW,velivckovic2017graph,Hamilton2017InductiveRL,Luo2024ClassicGA}, together with the CDS-inspired GraphCON-GAT~\cite{Rusch2022GraphCoupledON} and KuramotoGNN~\cite{Nguyen2023FromCO}. SIES attains the best results among all compared methods on four of the six benchmarks. The GraphCON-GAT comparison is especially diagnostic because it retains the same overdamped harmonic dynamics engine as SIES but uses conventional softmax attention. The non-negative attention coefficients can express attractive influence but cannot naturally encode repulsive interactions with the same dynamical semantics. By contrast, SIES uses learned signed couplings, allowing both homophilous attraction and heterophilous repulsion within the state-evolution law. Together, these results illustrate how scale-compatible coordination through local, signed interactions can be operationalized within a unified dynamical framework for synchronization control, adaptive robotic locomotion and heterophilous graph representation learning.

By unifying the two fundamental aspects of swarm intelligence, SIES enables graph-structured systems to generate richer collective behaviors while retaining strong dynamical properties. This integration directly addresses the complementary limitations of classical CDS and GNN frameworks identified earlier: CDSs gain learnable and optimizable couplings that generalize across system scales, while GNNs acquire dynamical semantics that naturally support both attraction and repulsion. More broadly, SIES offers a principled route toward understanding and engineering collective phenomena on graphs, in which local interaction rules can be designed to produce coherent, adaptable, and task-relevant global organization.

\section{Results}\label{sec2}

\subsection{SIES Reveals Generalizable Collective Dynamics} \label{sec:generalization}

\begin{figure}[t!]
\centering
\includegraphics[width=\textwidth]{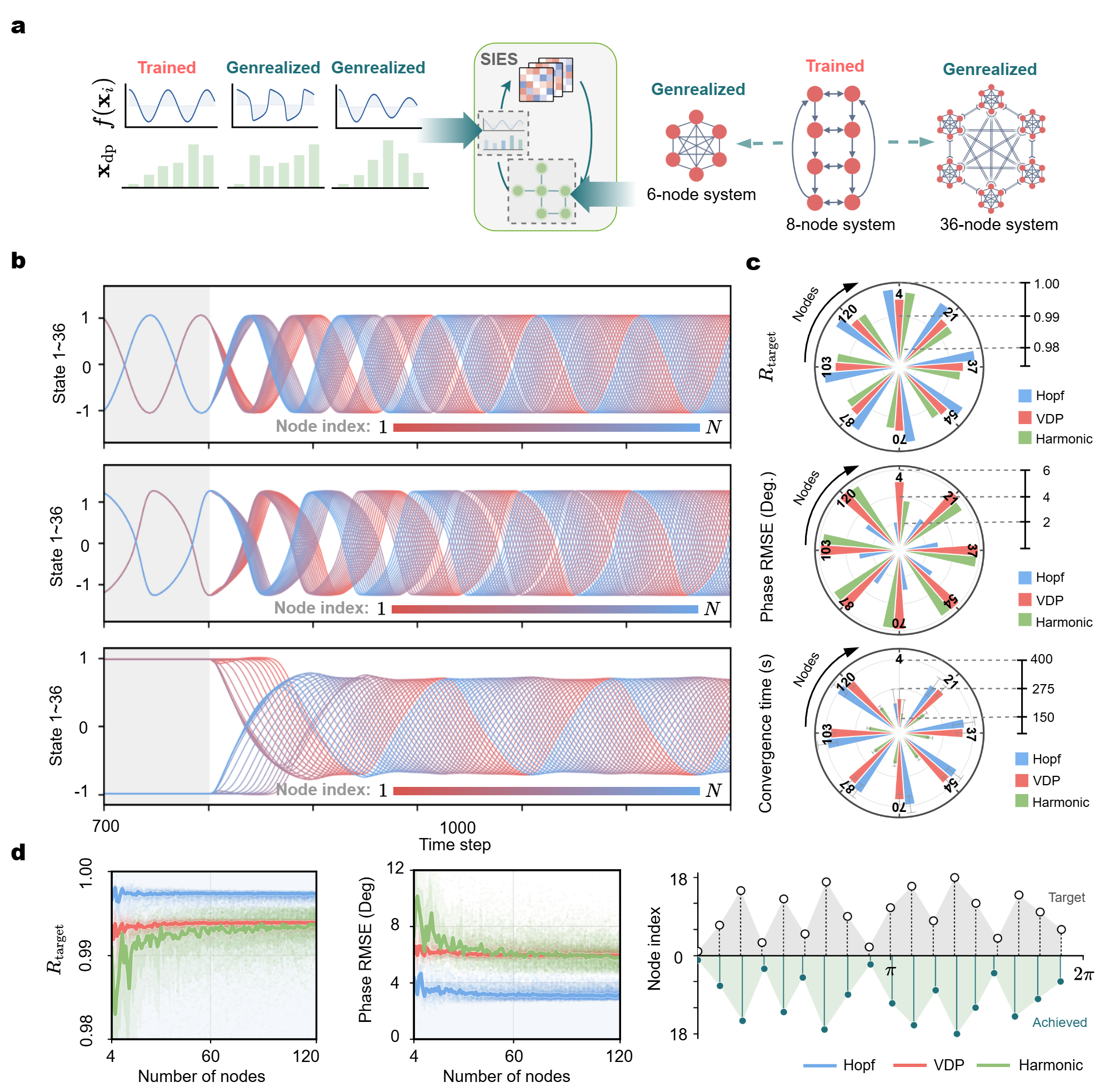}
    \caption{\textbf{Evaluation of SIES for synchronization control across unseen intrinsic node dynamics, system scales and synchronization targets.} \textbf{(a)} Evaluation design: a model trained on an 8-node Hopf system is tested on different intrinsic node dynamics, unseen objectives and CDSs with different numbers of nodes. \textbf{(b)} Representative 36-node traveling-wave outputs under Hopf, Van der Pol (VDP) and overdamped harmonic intrinsic node dynamics. Colors indicate oscillator-node index. \textbf{(c)} Target-aligned order parameter \(R_{\mathrm{target}}\), phase RMSE and convergence steps for traveling-wave targets as node number increases. \textbf{(d)} \(R_{\mathrm{target}}\) and phase RMSE for unseen random synchronization targets across node numbers, together with a representative comparison between target and achieved phases in an 18-node system.}
\label{fig:cross_scale_gen}
\end{figure}

Natural swarms can preserve coherent collective organization despite changes in group size, collective objective and the local dynamics of individual members. A swarm-inspired computational architecture should therefore be tested not only on whether it can form one trained target pattern, but also on whether the same local interaction rule remains effective under changes in network scale, intrinsic node dynamics and target patterns. We use SIES for synchronization control as a controlled dynamical probe of this question, because the desired collective state can be specified as a target phase configuration and quantified directly. We train SIES only on an 8-node coupled oscillator network with Hopf intrinsic node dynamics and four target phase configurations \(\mathbf{x}_{\text{dp}}\). Details are provided in Methods and Supplementary Information (SI) Section~\ref{app:sies_cem_rl}. We then evaluate this trained SIES under unseen network scales, synchronization targets, and intrinsic node dynamics (Fig.~\ref{fig:cross_scale_gen}a). Note that all evaluations in this experiment are conducted on fully connected CDSs. This section shows that the trained model generalizes across all three tested dimensions, and that a reduced phase model gives a local explanation for traveling-wave scale generalization.

\textbf{Generalization to unseen dynamics and scale.} SIES treats the graph of a CDS as an implicit input through its attention mechanism, allowing it to operate on graphs of different scales without modifying the SIES architecture. For each system scale \(N\), the traveling-wave target is defined by desired lags \(\theta_i=2\pi(i-1)/N\) relative to the reference oscillator (Methods, Section~\ref{sec:method_generalization}). The SIES model trained only on Hopf oscillators successfully generalizes to 36-node traveling-wave targets under unseen Van der Pol or overdamped harmonic dynamics (Fig.~\ref{fig:cross_scale_gen}b). \href{https://drive.google.com/file/d/1mxvNUCzE1o-S4vQsazRPx55vuTF_10dJ/view?usp=sharing}{SI Video 1} records the state evolution trajectories of networks of different scales under these different intrinsic node dynamics. We quantify target alignment using the target-aligned order parameter \(R_{\mathrm{target}}\), for which a value of one indicates perfect agreement with the desired phase configuration, and phase RMSE, which measures the remaining circular phase error after removing a global phase offset (SI Section~\ref{app:metrics}). Across the displayed network scales, traveling-wave generation remains ordered for all three intrinsic node dynamics, with high \(R_{\mathrm{target}}\), bounded phase RMSE and the convergence steps reported in Fig.~\ref{fig:cross_scale_gen}c.

\textbf{Generalization to unseen synchronization targets.} We next assign random target phase configurations that are unseen during training. At each system scale \(N\), each target pattern is sampled from \(\{2\pi p/N:p=0,\ldots,N-1\}\), with the reference node fixed at zero.  \(2N\) targets are evaluated per scale (Methods, Section~\ref{sec:method_generalization}). Target alignment remains high as scale increases, as shown by \(R_{\mathrm{target}}\) and phase RMSE for networks from 4 to 120 nodes (Fig.~\ref{fig:cross_scale_gen}d). The target-versus-achieved phase diagram further illustrates close alignment for a representative 18-node random target. Thus, SIES generalizes across both unseen intrinsic node dynamics and  synchronization targets.

\textbf{Insight into scale generalization.} The generalization experiments above show that SIES can generate prescribed organization across unseen system scales, suggesting that this empirical scale generalization may admit a theoretical explanation. The full learned SIES operator is too high-dimensional for direct analysis, so we examine a simplified phase model defined under weak-coupling and fully connected assumptions (Methods, Section~\ref{sec:simplified_phase_model}). For traveling-wave targets \(\theta_i=2\pi(i-1)/N\), Theorem~\ref{theo:1} shows that, for \(N>4\), a coupling function \(e_{ij}=A(\theta_i,\theta_j)\) can produce the target wave without being redesigned for each system scale. One construction is \(A(\theta_i,\theta_j)=\cos(\theta_j-\theta_i)+\sin(\theta_j-\theta_i)\), whose dynamics are illustrated in \href{https://drive.google.com/file/d/1w8vjVntMLjBdL83KlCsYNNXEnxuYEyOv/view?usp=sharing}{SI Video 2}. This result provides a local analytical explanation for scale-compatible traveling-wave generation, but it does not constitute a theory of the full learned operator or of random target patterns.

\subsection{SIES Combines Generalization with Convergence and Sparse Interactions}  \label{sec:convergence}

\begin{figure}[!t]
\centering
\includegraphics[width=0.9\textwidth]{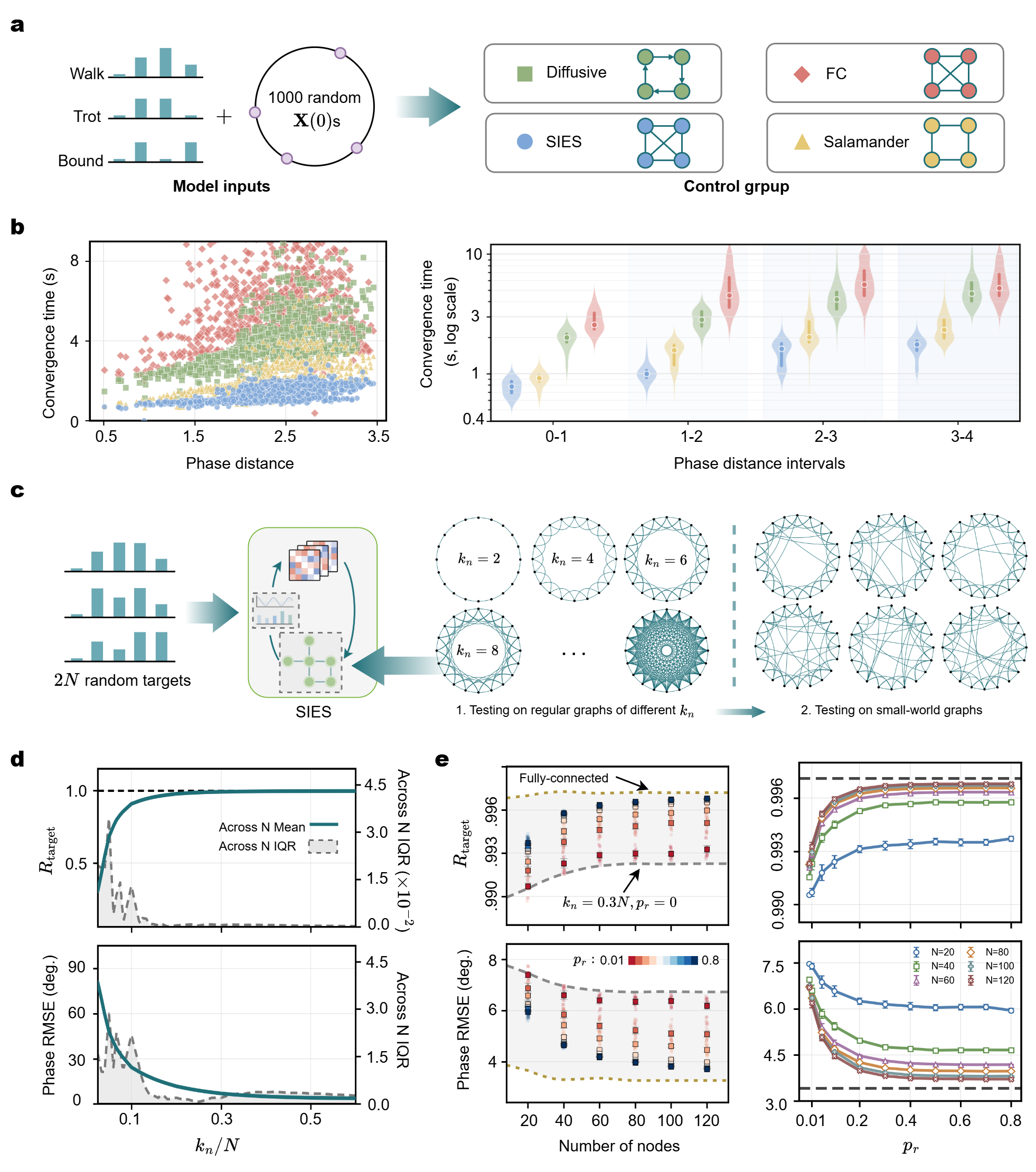}
    \caption{\textbf{Convergence and sparse-interaction evaluations of SIES for synchronization control.} \textbf{(a)} Experimental setup for comparing the convergence of SIES with three baseline models \textbf{(b)} Convergence evaluation for the walk target, showing convergence time for each initial condition (left) and distributions across phase-distance intervals (right, log scale). \textbf{(c)} Experimental setup of sparse-interaction evaluations of SIES using the small-world network model. For each network scale \(N=20,40,60,80,100,120\), \(2N\) random target synchronization patterns are tested on regular ring lattices with varying \(k_n\) and on small-world networks with fixed \(k_n=0.3N\) and varying rewiring probability \(p_r\). \textbf{(d)} Target-aligned order parameter \(R_{\mathrm{target}}\) (top) and phase RMSE (bottom) versus normalized connectivity \(k_n/N\) for regular ring lattices across network scales. \textbf{(e)} Impact of rewiring probability \(p_r\) on target synchronization performance in sparse small-world networks with fixed \(k_n=0.3N\). (Left) Metrics versus number of nodes. (Right) Metrics versus \(p_r\) for different network scales \(N\).}
\label{fig:basin_att}
\end{figure}

Natural swarms do more than produce global collective behavior; the process by which that behavior emerges is often stable, efficient and shaped by local adaptation. If SIES was to serve as a swarm-inspired synchronization model, target formation should therefore be examined together with reachability and convergence efficiency. We test whether the ability of SIES to form a prescribed target state also implies a broad empirical basin of attraction and optimized convergence dynamics. A second swarm principle concerns interaction sparsity: an individual agent does not need to attend to every other agent for stable collective behaviors to emerge. We therefore also test whether SIES can preserve ordered target patterns when the interaction graph is sparsified, rather than relying on dense all-to-all coupling. This subsection shows that SIES converges faster than three oscillator baselines in the tested gait settings and that dense all-to-all coupling is not required for the examined random phase targets.

Reliable convergence from diverse initial states is particularly important when a CDS is used to generate or transition between robotic gaits. We therefore compare SIES with three coupled-oscillator baselines on a 4-node network: a fully connected model (FC)~\cite{righetti2006design,Righetti2008PatternGW}, a nearest-neighbor model (Salamander)~\cite{ijspeert2007swimming} and a diffusively coupled model (Diffusive)~\cite{Yu2016GaitGW}. The three evaluated gait-related targets are trot \([0,\pi,\pi,0]\), walk \([0,\pi,3\pi/2,\pi/2]\) and bound \([0,\pi,0,\pi]\) (SI Section~\ref{sec:comp_sies_cem}).

Across 1000 sampled initial conditions shared across all models, we measured convergence to gait-related phase configurations. Convergence time is the earliest time at which the instantaneous target-aligned order parameter reaches \(R_{\mathrm{target}}(t)=0.999\) (Eq.~\eqref{eq:convergence_time}). Initial-state separation is measured by the circular phase-distance metric (SI Eq.~\eqref{eq:PD}). Fig.~\ref{fig:basin_att}a summarizes the convergence-test setup. \href{https://drive.google.com/file/d/190nEunSvKkB5vvHRMQkHIUsh01Hp-TZp/view?usp=sharing}{SI Video 3} intuitively visualizes the state evolution trajectories of the four models under different initial conditions. For the walk target, the trial-level scatter plot shows a compact, approximately monotonic increase in SIES convergence time with respect to initial phase distance. The SIES points remain below the broader and more variable baseline clouds across the displayed range (Fig.~\ref{fig:basin_att}b, left, Extended Data Fig.~\ref{fig:basin_att_ext}a). When the trials are grouped by phase-distance interval and plotted on a logarithmic scale, the median convergence time of SIES is lower than that of all three baselines in every interval (Fig.~\ref{fig:basin_att}b, right, Extended Data Fig.~\ref{fig:basin_att_ext}b). The corresponding cumulative convergence curve reaches the displayed \(50\%\), \(80\%\) and \(95\%\) trial fractions earlier for SIES than for Salamander, Diffusive and FC (Extended Data Fig.~\ref{fig:basin_att_ext}c). The baseline-minus-SIES convergence-time differences remain positive across phase-distance intervals (Extended Data Fig.~\ref{fig:basin_att_ext}d). These complementary views show that the advantage is not confined to initial states that are already close to the target configuration. Unlike the baseline models, whose relative performance varies across different gaits, SIES maintains stable and superior convergence across targets. Taken together, these results demonstrate that reinforcement learning improves not only the diversity of achievable synchronization modes in SIES, but also the dynamical process by which those modes are reached. Compared with the baselines, SIES exhibits both a broader empirical basin of attraction and faster, more consistent convergence across varying phase distances and target patterns. This dual improvement in mode richness and convergence efficiency closely parallels the behavior of natural swarms, where collective patterns emerge not merely as final states, but through stable, efficient, and locally adaptive dynamical processes. By jointly optimizing target formation and the dynamics of reaching those targets, SIES provides a more faithful swarm-inspired model of synchronization control, in which both the destination and the trajectory toward it are shaped by learned local interactions.

 We next investigate how interaction density and topological randomness govern the ability of SIES to realize sampled unseen target phase configurations by employing the small-world network model~\cite{watts1998collective} (Fig.~\ref{fig:basin_att}c). We first fix the rewiring probability at $  p_r = 0  $ (reducing the topology to regular ring lattices) and, for each network scale $  N = 20, 40, 60, 80, 100, 120  $, vary the nearest-neighbor count $  k_n  $ from 2 to $  N-2  $ in steps of 2. For every $  (N, k_n)  $ pair we evaluate $  2N  $ randomly sampled target phase patterns. As shown in Fig.~\ref{fig:basin_att}d, both the target-aligned order parameter \( R_{\mathrm{target}} \) and phase RMSE (each computed as the mean across network scales \( N \in \{20, 40, 60, 80, 100, 120\} \)) improve rapidly with increasing normalized connectivity \( k_n/N \). \href{https://drive.google.com/file/d/124i1GDCNVXZr1qIrF5WhRodufIxw2_3i/view?usp=sharing}{SI Video 4} intuitively visualizes the state evolution trajectory of a 72-node CDS network as \( k_n \) increases, showing that a relatively stable traveling wave already emerges at \( k_n = 22 \). The low interquartile range across these different scales indicates that the trend is highly consistent regardless of network scale. Notably, both metrics approach the performance of fully connected networks once \( k_n/N \approx 0.3 \), with \( R_{\mathrm{target}} \) rising sharply toward 1 and phase RMSE decreasing accordingly. These results demonstrate that SIES can generate diverse target phase patterns with high fidelity even on relatively sparse interaction graphs.

To examine whether controlled topological randomness can further enhance performance under sparse connectivity, we then fix $  k_n = 0.3N  $ and systematically increase the rewiring probability $  p_r  $. For each $  p_r  $ we generate 20 independent small-world networks and test every network on the same $  2N  $ random targets. The outcomes (Fig.~\ref{fig:basin_att}e) reveal a counter-intuitive yet consistent trend: raising $  p_r  $ steadily improves both $  R_{\mathrm{target}}  $ and phase RMSE across all tested network scales. Notably, for the largest scales ($  N=100  $ and $  N=120  $), the performance of these sparse networks with moderate rewiring nearly recovers that of the fully connected topology. This result suggests that the long-range shortcuts introduced by rewiring facilitate global coordination, allowing SIES to achieve near-optimal performance even under sparsity.

These sparse-interaction findings underscore a central strength of the SIES framework: its ability to support rich target synchronization patterns through sparse interactions that range from locally structured ring lattices to largely unstructured, random-like topologies, rather than requiring dense coupling. The observation that regular ring lattices already achieve near-optimal performance at normalized connectivity around 0.3 directly echoes a fundamental characteristic of natural swarms, where coordinated collective dynamics emerge without requiring all-to-all interactions. Increasing the rewiring probability steadily improves both target alignment and phase accuracy, enabling sparse networks to approach the synchronization performance of fully connected systems even at large scales. These results indicate that SIES captures an efficient organizational principle of swarm coordination in which controlled topological randomness enhances rather than disrupts the emergence of target patterns under sparsity.

\subsection{Signed Direction-Aware Coupling Underlies the SIES Mode Repertoire}\label{sec:result_ablation}

To determine which components of the SIES architecture enable access to phase-lagged collective modes, we evaluate the full synchronization-control model against three matched ablation variants on an 8-node fully connected oscillator system. Performance is assessed using held-out diagnostic phase targets. Detailed experimental configurations, including model variants, training protocols, and evaluation metrics, are provided in Methods Section~\ref{sec:method-ablation}.

Replacing signed attention with standard non-negative softmax attention prevents the formation of phase-lagged modes on held-out targets, showing that repulsive couplings are required to access the observed synchronization repertoire. Tying the source and target conditioning projections while preserving signed attention allows the model to fit training targets but degrades performance on held-out patterns that require asymmetric coupling. Aggregating messages in the original state space rather than the learned feature space maintains mode accessibility but slows learning and lowers final accuracy. Detailed experimental results are presented in SI Section~\ref{sec:mechanistic_cem}.
Together, these controls demonstrate that signed coupling combined with separate source-target conditioning forms the expressive core enabling diverse synchronization patterns, whereas feature-space aggregation mainly improves optimization speed and mode fidelity.

\subsection{SIES Enables Topology-Aware Rhythmic Control in Multi-Legged Robots}\label{sec:robot-exp}

\begin{figure}[t!]
\centering
\includegraphics[width=.95\textwidth]{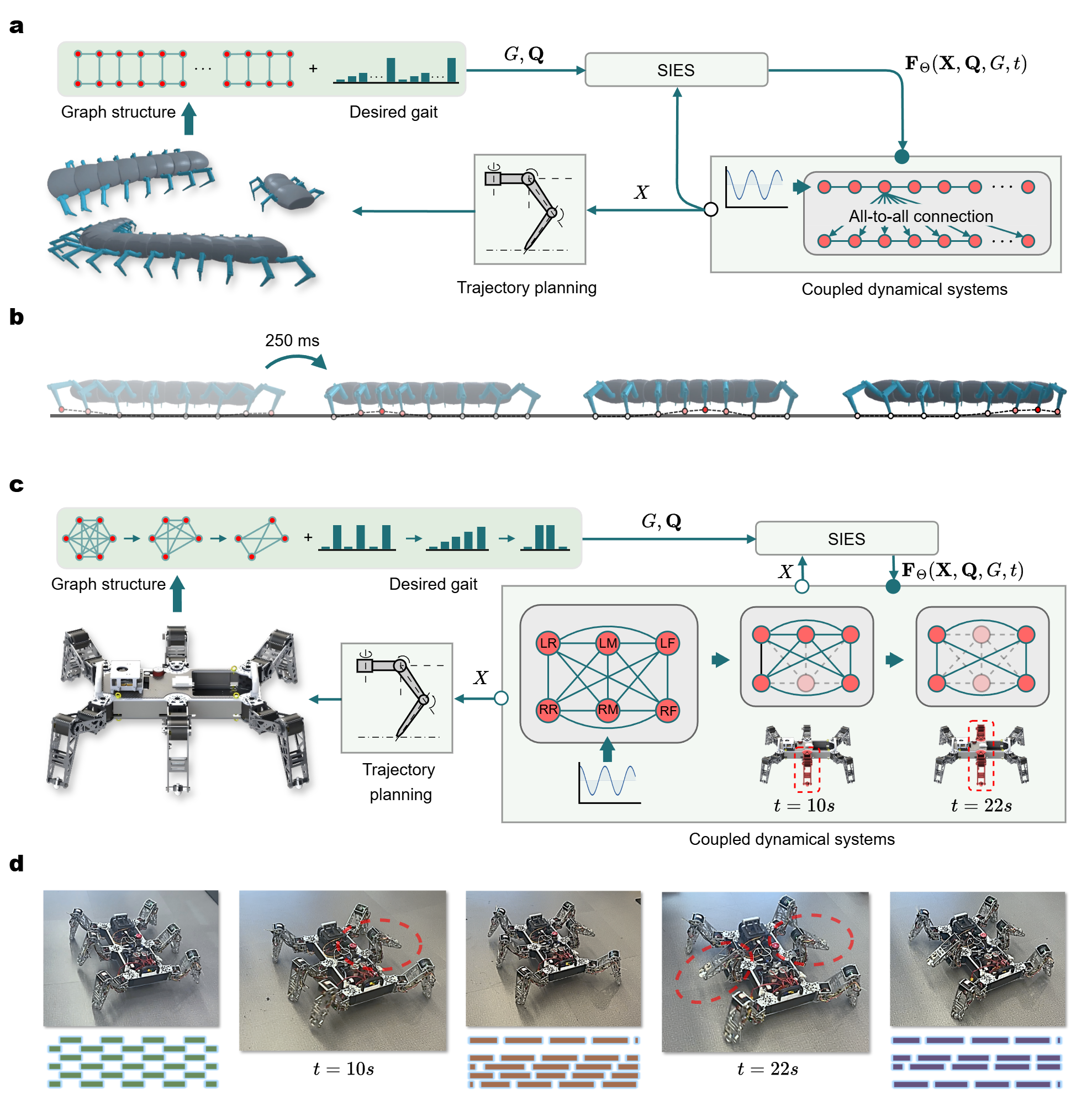}
    \caption{\textbf{Topology-aware rhythmic control with SIES in multi-legged robots.} \textbf{(a)} Control framework for simulated centipede-like robots. SIES receives the leg-foot topology of robots with different segment numbers, constructs the corresponding CDS network, computes coupling terms online and generates rhythmic signals for trajectory planning and joint-command generation. \textbf{(b)} Consecutive snapshots of forward locomotion in an 8-segment simulated centipede robot. \textbf{(c)} Control framework for a physical hexapod under time-varying active-leg topology. Sequential leg disablement changes the active CDS network from six to five and then four nodes, while SIES updates the rhythmic signals supplied to the trajectory-planning module. \textbf{(d)} Experimental snapshots and corresponding foot-contact patterns before disablement, after disabling one leg at \(t\approx10\) s, and after disabling a second leg at \(t\approx22\) s.}
\label{fig:fault_tolerant}
\end{figure}

The preceding results show that SIES generalizes learned synchronization dynamics across system scale, target phase relations, sparse interaction graphs, and intrinsic oscillator dynamics. The next question is whether these graph-conditioned collective dynamics remain useful when the nodes are embedded in a robot body. This test is stricter than pure oscillator simulation because the generated phases must be converted into joint commands, and the relevant graph is the leg-foot topology of an embodied system rather than an abstract benchmark. Multi-legged locomotion provides a natural testbed: coordinated gait requires stable non-in-phase rhythms, while morphology changes or leg disablement demand rapid reorganization of the active coordination graph. We therefore test whether the trained SIES model used in the above section can serve as a topology-aware rhythmic generator for simulated centipede-like robots with different numbers of body segments and for a physical hexapod after sequential leg disablement.

\textbf{Gait simulation in centipede-like robots.} We deploy the trained SIES model in PyBullet simulations of centipede-like robots with 6, 16 and 32 legs (3, 8 and 16 segments). Fig.~\ref{fig:fault_tolerant}a summarizes the simulated rhythmic-control framework. For each body scale, we provide the leg-foot topology to SIES as the interaction graph. SIES constructs the corresponding oscillator CDS, computes the learned coupling terms online, and generates one rhythmic signal per leg. A trajectory-planning module then converts these rhythmic signals, together with duty-cycle and steering inputs, into foot trajectories and joint commands (Methods, Section~\ref{sec:method_robot_locomotion}; SI Section~\ref{app:motion_gen}). The desired phases follow
\begin{equation}
    \begin{aligned}
    \mathbf{x}_{\text{dp}} &= 2\pi\left[\boldsymbol{\vartheta}_N,\boldsymbol{\vartheta}_N\right], \\
    \boldsymbol{\vartheta}_N &= \left[0,\frac{2}{N},\frac{4}{N},\ldots,\frac{N-2}{N}\right],
    \end{aligned}
    \label{eq:centipede_target}
\end{equation}
where the repeated vectors specify the corresponding phase progressions on the two sides of the body and produce rear-to-front caterpillar gaits. Fig.~\ref{fig:fault_tolerant}b shows consecutive snapshots from an 8-segment simulated centipede, illustrating the rhythmic signals after conversion by the trajectory-planning module. \href{https://drive.google.com/file/d/1SjK_hP2fbSZ_da0YFd8VK8w7Wy_phddv/view?usp=sharing}{SI Video 5} presents the full caterpillar-like gaits produced by the trained SIES model for the 3-segment, 8-segment and 16-segment centipede-like robots, as well as turning maneuvers of the 16-segment robot under the same gait. These demonstrations show that SIES operates as a reusable, topology-aware rhythmic generator that can be directly deployed across multi-legged robots with diverse body plans. In centipede robotics, a central objective has been to understand how morphological parameters—particularly segment number and body-axis flexibility—shape locomotor performance and maneuverability~\cite{hoffman2012turning,aoi2016advantage,diaz2023active}. Prior studies on modular centipede-inspired robots have shown that passive body undulations enhance both forward locomotion and turning across varying segment counts, and that body flexibility can induce beneficial dynamic instability that facilitates rapid turning maneuvers. However, most existing generators are hand-designed for specific morphologies, limiting systematic investigation across different body plans. SIES addresses this limitation by providing a single, topology-aware pattern generator that produces coordinated rhythmic outputs for robots with diverse segment numbers and interaction graphs. Because the same trained model operates without retraining or morphology-specific coupling rules, it enables efficient exploration of how segment number and flexibility jointly influence gait generation and turning agility in both simulated and physical multi-legged systems.

\textbf{Adaptation to leg disablement in a physical hexapod.} We further evaluate the trained SIES model on a physical hexapod robot using a similar rhythmic generator architecture (Fig.~\ref{fig:fault_tolerant}c; Methods, Section~\ref{sec:method_robot_locomotion}; SI Sections~\ref{app:robots} and~\ref{app:motion_gen}). In this experiment, the leg-foot topology changes over time because we sequentially disable the second and fifth legs at \(t \approx 10\) s and \(t \approx 22\) s. After each disablement, SIES receives the updated active-leg graph, constructs a six-, five- or four-node CDS, and computes coupling terms for the remaining leg oscillators. The resulting rhythmic signals pass to the trajectory-planning module to produce joint commands for the active legs. The corresponding phase-lag targets are tripod \([0,\pi,0,\pi,0,\pi]\), five-leg metachronal \([0,2\pi/5,4\pi/5,6\pi/5,8\pi/5]\) and trot \([0,\pi,\pi,0]\), with duty cycles specified in Methods.

The experiment demonstrates smooth transitions from a six-legged tripod gait to a five-legged metachronal gait and then to a four-legged trot gait after each graph update (Fig.~\ref{fig:fault_tolerant}d and \href{https://drive.google.com/file/d/1Ep98PGMzkXJPSx9rxquujfvTlDjxYf8Y/view?usp=sharing}{SI Video 6}). We use a metachronal target for the asymmetric five-legged condition. Fig.~\ref{fig:fault_tolerant}d also reports the corresponding foot-contact patterns across the three stages. This result shows that the trained SIES model remains executable on altered embodiment graphs in real time without retraining a separate coupling model for each active-leg configuration. In resilient robotics, a central objective is to enable multi-legged robots to maintain effective locomotion after partial morphological damage, such as leg loss, while achieving graceful degradation and minimizing dependence on pre-designed controllers for specific failure modes~\cite{cully2015robots}. In this context, SIES offers a distinct approach to resilient locomotion by treating morphological changes as updates to the interaction graph of a coupled dynamical system. Instead of searching for new behaviors after damage or maintaining separate controllers for different leg configurations, the same trained model recomputes the coupling relationships among the remaining functional oscillators. This enables real-time reconfiguration of rhythmic coordination directly from the current embodiment topology. 


\subsection{SIES Extends Signed Dynamical Attention to Heterophilous Graph Learning}\label{sec:graph-learning}

The final question is whether the SIES mechanism is also useful for graph neural networks. Conventional message passing tends to favor homophilous attraction, which can be limiting on heterophilous graphs where useful representations may require unlike neighbors to separate rather than collapse into a common embedding~\cite{Chamberlain2021GRANDGN,Rusch2022GraphCoupledON,Nguyen2023FromCO,Platonov2023ACL}. We view the GNN classification task as that of rotating features of nodes from different classes into distinct regions of a latent phase space, such that a downstream decoder can discriminate node categories according to the phase encoding of each node. Our earlier results on SIES for synchronization control establish that this mechanism generates diverse phase differences among nodes even in networks of arbitrary scale. Most critically, because SIES computes signed attention within a coupled dynamical system, it naturally attracts nodes of the same class while repelling those of different classes. We therefore evaluate whether SIES can alleviate this homophilous attraction bias by allowing both attractive and repulsive neighbor influences inside a CDS-based representation-learning model. In this setting, node features initialize the dynamical states and provide the node-wise conditioning input, while classification loss optimizes the signed coupling without prescribed phase targets (Methods, Section~\ref{sec:sies_gnn_method}).

\begin{figure}[t!]
\centering
\includegraphics[width=\textwidth]{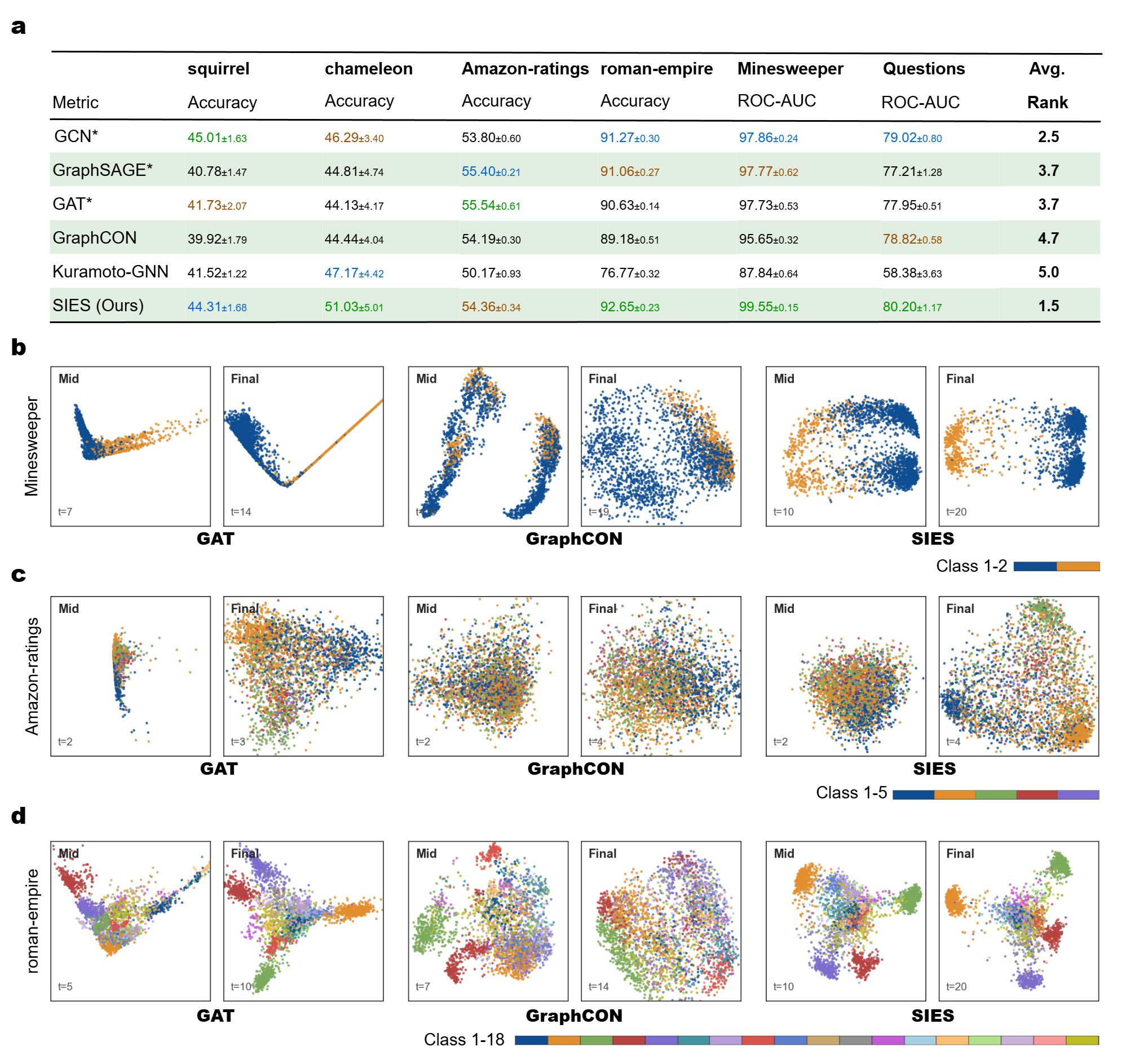}
\caption{\textbf{SIES performance and dynamical representations on heterophilous graph node classification.} 
\textbf{(a)} Node-classification performance of enhanced conventional GNNs, CDS-based baselines and SIES on six heterophilous datasets. Values are reported as means ± standard deviations. The top-1, top-2 and top-3 results are highlighted in green, blue and brown, respectively. 
\textbf{(b--d)} PCA projections of node states at intermediate and final integration steps for GAT, GraphCON and SIES on Minesweeper, Amazon-Ratings and Roman-Empire, respectively. Points are colored by class label.}
\label{fig:gnn_evolution}
\end{figure}

We evaluated SIES on six heterophilous graph benchmarks: Amazon-ratings, roman-empire, Minesweeper, Questions, squirrel and chameleon (Methods, Section~\ref{sec:method_gnn_datasets}). Comparators comprised enhanced conventional GNN baselines (GAT*, GCN* and GraphSAGE*) reported with architecture enhancements following~\cite{Luo2024ClassicGA}, and two CDS-inspired models, GraphCON~\cite{Rusch2022GraphCoupledON} and KuramotoGNN~\cite{Nguyen2023FromCO}. We use accuracy for multi-class datasets and ROC-AUC for binary datasets, reporting the test score of the validation-selected model as mean and standard deviation over the prescribed splits (SI Eqs.~\eqref{eq:accuracy} and~\eqref{eq:roc_auc}). Detailed experimental settings and hyperparameter search ranges are specified in Methods.

Fig.~\ref{fig:gnn_evolution}a summarizes classification performance across the six heterophilous benchmarks. Enhanced conventional GNNs achieve the highest displayed results on Squirrel and Amazon-Ratings, whereas SIES attains the highest performance on Chameleon (51.03\%), Roman-Empire (92.65\%), Minesweeper (99.55\% ROC-AUC), and Questions (80.20\% ROC-AUC). Notably, Minesweeper possesses a highly regular graph structure and is widely regarded as a dataset particularly well-suited to standard GNN architectures. Consequently, enhanced baselines have already reached a strong performance level, typically plateauing around 97\% ROC-AUC. SIES nevertheless exceeds the performance of the compared baselines on this dataset and reach 99.55\% ROC-AUC indicates that the model can more thoroughly exploit the intrinsic structure of datasets that align well with its underlying dynamical principles.

GraphCON-GAT provides a controlled non-negative-attention counterpart to SIES in our graph-learning evaluation. Both models use the same damped harmonic CDS backbone and matched experimental protocol; their defining operator difference is that GraphCON-GAT applies the standard GAT softmax map, producing non-negative node attention, whereas SIES retains signed coupling coefficients (Eq.~\eqref{eq:gnn_sign_control}). A CDS backbone can provide continuous representation dynamics, as in GraphCON, but a non-negative softmax map cannot directly express repulsive neighbor influence. Conversely, signed coefficients acquire their interpretation as attraction or repulsion most naturally when they enter a state-evolution equation rather than a static aggregation rule. The comparison therefore tests whether the signed-attention coupling formulation improves on a literature-grounded non-negative attention operator within a matched oscillator-GNN setting, and whether CDS dynamics can express both neighbor attraction and neighbor repulsion during graph learning.

To better understand how different architectures shape node representations, we examine principal-component projections of node states during inference (Figs.~\ref{fig:gnn_evolution}b--d). Conventional attention-based models such as GAT primarily operate through a diffusive mechanism: they gradually push nodes of different classes toward distinct regions or hyperplanes in feature space. In contrast, the CDS-based models generate evolving representation trajectories over the integration horizon, causing them to spread and oscillate across different phase regions rather than collapsing into static clusters. However, GraphCON-GAT lacks an explicit mechanism for modeling repulsive interactions between nodes. As a result, many node representations remain distributed across intermediate regions of the feature space. SIES, by incorporating signed attention that explicitly encodes both attractive and repulsive couplings, enables nodes to be more effectively repelled from dissimilar classes while being attracted to similar ones. Consequently, the final node states exhibit clearer class-associated structures on several benchmarks, most notably on Minesweeper.

Together, these experiments show that the signed dynamical-coupling mechanism of SIES is not restricted to prescribed synchronization control. In graph representation learning, the same signed-coefficient formulation can be embedded in CDS evolution to produce task-relevant node representations on heterophilous benchmarks. The matched comparison with GraphCON-GAT indicates that the benefit is not attributable to the overdamped harmonic backbone alone, because replacing signed degree-normalized coefficients with non-negative softmax attention removes the direct representation of repulsive neighbor influence. These results support SIES as a CDS-based GNN mechanism that can reduce homophilous attraction bias when heterophilous separation is useful for node classification.

\section{Discussion}

The central finding of this study is that SIES provides a minimal computational realization of the core mechanisms that enable natural swarm systems to generate rich, scalable, and generalizable collective behaviors. Through a single learned local interaction rule that adaptively determines signed (attractive or repulsive) couplings between neighboring nodes, SIES spontaneously produces a wide spectrum of emergent functional signatures characteristic of biological swarms. These include coordinated collective rhythms that generalize across untrained system scales, unseen phase targets, and altered intrinsic node dynamics; faster empirical convergence relative to conventional oscillator baselines; robust preservation of target organization under sparse and irregular interaction graphs; and direct transfer to physical robot body graphs whose scale and connectivity change dynamically. The same local signed-interaction mechanism further improves node representation learning on heterophilous graphs, where effective organization requires both attraction and repulsion. Crucially, all of these capabilities arise without a centralized controller or hand-specified global pattern generator. Each node uses only local state and task information to decide whether neighboring states should attract or repel it, and iterated coupled dynamical updates convert these local signed decisions into coherent system-level organization. Source-target conditioning supplies directional task context, signed coefficients provide the attraction–repulsion vocabulary, and feature-space aggregation improves the expressivity of learnable modes. In this sense, SIES demonstrates that a minimal local rule, when equipped with adaptive signed interactions, is sufficient to reproduce the scalable and dynamically generalizable emergent behaviors that define natural swarm intelligence.

This mechanism distinguishes SIES from prior CDS and GNN approaches. Classical CDS models provide explicit dynamics and interpretable coupling, but many constructions are designed around selected phase patterns, network symmetries or robot morphologies~\cite{collins1993coupled,Buono2001ModelsOC,stewart2003symmetry,golubitsky2005patterns}. Gait-oriented oscillator models, including fully connected, nearest-neighbor and diffusive couplings~\cite{righetti2006design,Righetti2008PatternGW,ijspeert2007swimming,Yu2016GaitGW}, can generate prescribed coordination patterns, but the present comparisons show that such hand-designed couplings do not automatically optimize reachability from sampled initial states. SIES addresses this CDS bottleneck by learning the coupling term over dynamical rollouts, allowing target fidelity and convergence behavior to enter the optimization objective directly. Conversely, conventional GNNs are highly trainable but often aggregate neighbor information through non-negative weights, reinforcing homophilous attraction and limiting their ability to represent repulsive relations on heterophilous graphs~\cite{Chamberlain2021GRANDGN,Rusch2022GraphCoupledON,Nguyen2023FromCO,Platonov2023ACL}. The matched GraphCON-GAT comparison is therefore important: when the damped harmonic backbone is held fixed, replacing signed degree-normalized coefficients with softmax attention removes the direct dynamical representation of repulsive neighbor influence. Thus, the novelty of SIES is not merely combining a neural network with a dynamical system, but assigning learned signed attention the role of a CDS coupling law.

This study has several limitations. First, although SIES achieves high target alignment, it remains an approximate learned controller: the generated synchronization patterns can retain residual phase errors of a few degrees, typically around \(2^\circ\)--\(4^\circ\), and the convergence results support a broad empirical basin over sampled initial states rather than a guarantee of global reachability. Applications requiring strict phase accuracy or certified transient behavior may therefore need additional feedback correction, stability-constrained training, or safety-filtered coupling updates. Second, the current analytical result is limited to scale-compatible traveling-wave generation in a reduced weak-coupling, fully connected phase model. It does not yet provide stability, reachability, or error bounds for the full multi-head feature-space SIES operator, random targets, sparse or directed graphs, heterogeneous node dynamics, or topology changes during transients. Third, the robotic experiments demonstrate topology-aware execution rather than a complete locomotion robustness benchmark. The simulated centipede and physical hexapod results show that the same trained coupling rule can be executed across different body graphs and after leg disablement, but systematic evaluations of speed, energy cost, disturbance rejection, terrain variation, onboard deployment, and repeated-trial reliability remain to be conducted. Finally, while SIES is executed through graph-local coupling at inference time, the current synchronization-control model is trained offline using global task objectives. Developing more decentralized or online adaptation procedures remains an important step toward closer alignment with biological swarm intelligence.

Future work should develop SIES in three directions. First, theory should extend the current traveling-wave analysis toward graph-dependent conditions for stability, reachability, bounded synchronization error, and scale generalization. Combining phase reduction, Lyapunov or contraction analysis, signed-graph spectra, and approximation bounds for learned attention maps could make SIES-like models more certifiable graph-dynamical mechanisms. Second, SIES should be advanced as a morphology-aware coordination principle for embodied systems. The centipede and leg-disabled hexapod experiments provide an initial demonstration that learned coupling laws can remain executable as body morphology and active-leg topology change, connecting SIES to resilient and self-organized robot locomotion~\cite{bongard2006resilient,cully2015robots,zhang2017resilient,owaki2017quadruped,thandiackal2021emergence}. Future work should integrate sensory feedback, online topology inference, and constrained adaptation to test robustness, controllability, energy efficiency, and recovery across morphologies, terrains, contacts, and damage conditions. Third, SIES should be explored as a general learnable coupling operator for graph-dynamical modeling. Because many natural and engineered systems derive function from local interactions, SIES may be useful beyond synchronization and heterophilous graph learning, including spatial forecasting over coupled variables~\cite{lorenz1963deterministic,bauer2015quiet,keisler2022forecasting,lam2023graphcast}, adaptive sensor or antenna arrays~\cite{balanis2016antenna,hansen2009phased}, and coupled resonator networks~\cite{borra2017dynamics,borra2023multiple}. These studies should combine domain-specific constraints, mechanistic analysis, and rigorous validation, so that SIES becomes both a transferable architecture and a principled tool for engineering coupling-driven dynamics.

\section{Methods}\label{sec:methods}

\subsection{SIES Basic Formulations}
For a network of $N$ coupled dynamical nodes, let \(\mathbf{X}=[\mathbf{x}_1,\ldots,\mathbf{x}_N]^\top\) collect node states, let \(G\) denote the interaction graph, and let \(\mathbf{Q}=[\mathbf{q}_1,\ldots,\mathbf{q}_N]^\top\) denote node-wise task-conditioning inputs. The state $\mathbf{x}_i \in \mathbb{R}^m$ of each node evolves according to its intrinsic node dynamics $f: \mathbb{R}^m \to \mathbb{R}^m$ and the node-specific output of the learned coupling operator:
\begin{equation}
    \dot{\mathbf{x}}_i = f(\mathbf{x}_i) +
    \left[\mathbf{F}_{\Theta}(\mathbf{X},\mathbf{Q},G)\right]_i
    \equiv f(\mathbf{x}_i) + \mathbf{a}_i, \quad i=1,2,\dots,N.
    \label{eq:sies_base}
\end{equation}

Here, \(\Theta\) collects all trainable parameters, and the coupling input to node \(i\) is given by \(\mathbf{a}_i = [\mathbf{F}_{\Theta}(\mathbf{X}, \mathbf{Q}, G)]_i\). As illustrated in Fig.~\ref{fig:sies_update}, SIES computes this coupling by aggregating information from neighboring nodes through learned attractive or repulsive attention weights.

\begin{figure}[t!]
\centering
\includegraphics[width=.8\textwidth]{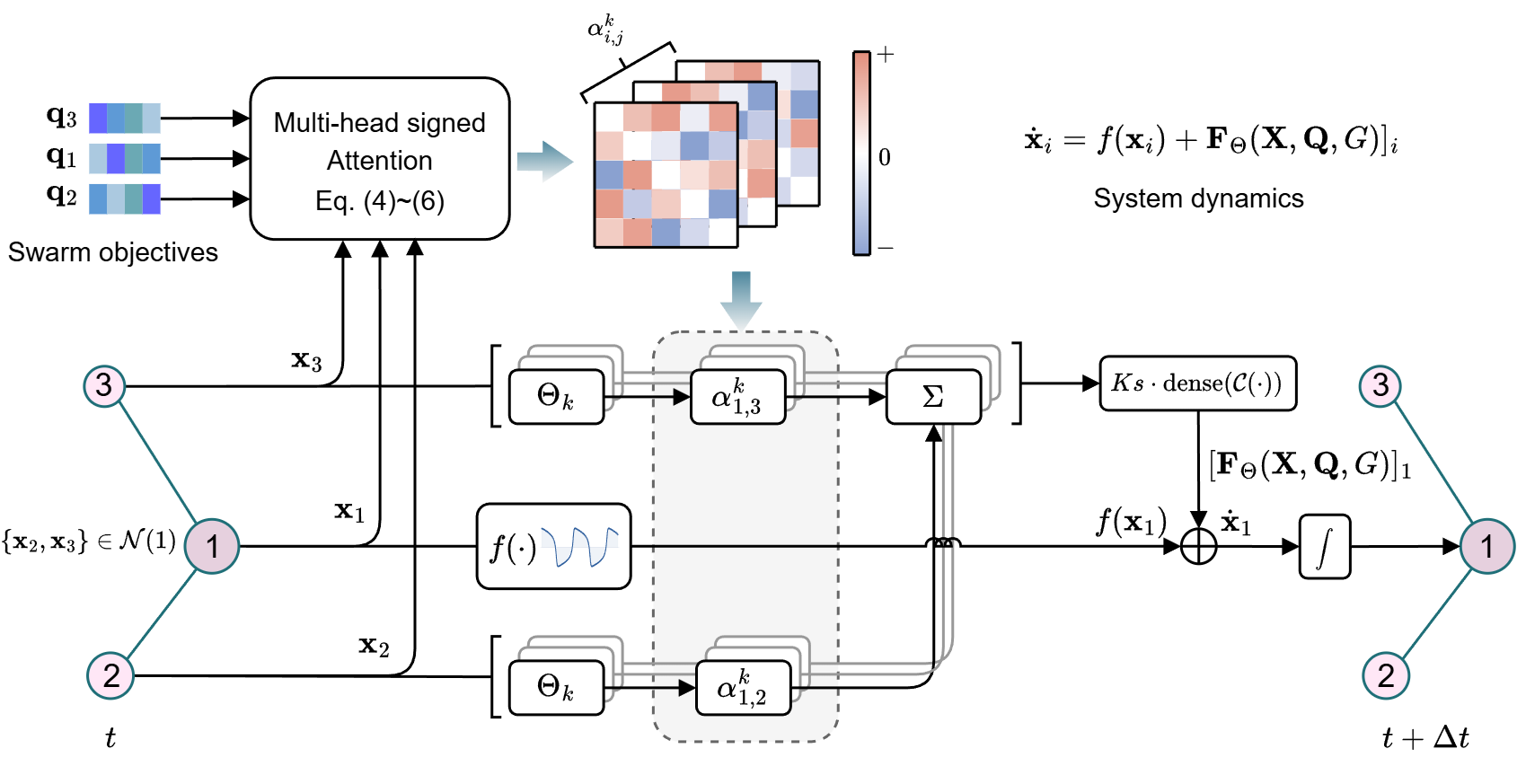}
    \caption{\textbf{SIES computation process for a single step of node state update.} In each update step, SIES computes task-conditioned multi-head attention coefficients to aggregate states from neighboring nodes in the coupled dynamical system. The resulting coupling signal is then integrated with the intrinsic node dynamics of each node to produce the updated node states.}
\label{fig:sies_update}
\end{figure}

\subsubsection{Model Construction}

We consider coupled oscillator systems and use three types of planar node dynamics \(f: \mathbb{R}^2 \to \mathbb{R}^2\) for training and evaluating SIES. The Hopf oscillator~\cite{hopf1942abzweigung} (SI Eq.~\eqref{eq:intrinsic}), which has a circular limit cycle, is used to train SIES for synchronization control. The Van der Pol oscillator~\cite{van1926lxxxviii, Yu2016GaitGW} (SI Eq.~\eqref{eq:intrinsic_vdp}) and the damped harmonic oscillator (SI Eq.~\eqref{eq:intrinsic_overdamped}), which exhibit elliptical limit-cycle and fixed-point dynamics, respectively, are used for generalization evaluation. Damped harmonic dynamics are also used for graph representation learning.

To expand the node-specific output \(\mathbf{a}_i\) of the operator in Eq.~\eqref{eq:sies_base}, we adopt a message-passing scheme inspired by graph neural networks. The node connectivity in \(G\) is encoded by the graph adjacency matrix \(A\), with \(\mathcal{N}(i)\) denoting the neighborhood of node \(i\). The SIES model computes \(\mathbf{a}_i\) as a weighted aggregation of neighbor states:
\begin{equation}
    \mathbf{a}_i = K_s \cdot \text{dense}\Bigl(
    \mathcal{C}\Bigl(
    \bigl\{
    \sum_{j \in \mathcal{N}(i)} \alpha_{i,j}^k \Theta_k \mathbf{x}_j
    \bigr\}_{k=1}^K
    \Bigr)
    \Bigr),
    \label{eq:sies_coupling}
\end{equation}
where \(\Theta_k \in \mathbb{R}^{F \times 2}\) is a learnable projection matrix contained in \(\Theta\), mapping the 2-dimensional state to an \(F\)-dimensional feature space, \(K_s > 0\) is the coupling strength, and \(K\) is the number of attention heads. The operator \(\mathcal{C}\) combines multi-head outputs and has two variants:
\begin{itemize}
    \item Averaging: \(\mathcal{C}_{\text{avg}}\bigl(\{\mathbf{h}^k\}_{k=1}^K\bigr) = \frac{1}{K}\sum_{k=1}^K \mathbf{h}^k\),
    \item Concatenation: \(\mathcal{C}_{\text{concat}}\bigl(\{\mathbf{h}^k\}_{k=1}^K\bigr) = \bigl[\mathbf{h}^1 , \mathbf{h}^2 , \dots, \mathbf{h}^K\bigr]\),
\end{itemize}
with \(\mathbf{h}^k = \sum_{j \in \mathcal{N}(i)} \alpha_{i,j}^k \Theta_k \mathbf{x}_j\). The subsequent dense layer (linear transformation without activation) maps the aggregated features back to \(\mathbb{R}^2\): a \(2 \times F\) matrix for averaging and a \(2 \times KF\) matrix for concatenation.

SIES uses a common graph-conditioned coupling operator in coupled dynamical systems and graph neural networks. The oscillator, robotic and node-classification experiments test this operator in the distinct settings described in the main text.

\subsubsection{Computation of Attention Coefficients}\label{sec:cs}

Let \(\mathbf{m}_i^\text{s}\) and \(\mathbf{m}_j^\text{t}\) denote the source and target auxiliary features for nodes \(i\) and \(j\). The raw attention score for the \(k\)-th head is computed as
\begin{equation}
    e_{i,j}^k = \langle \mathbf{a}^k, \ \text{LeakyReLU}(\mathbf{m}_i^\text{s} + \mathbf{m}_j^\text{t}) \rangle,
    \quad j \in \mathcal{N}(i),
    \label{eq:sies_raw_att}
\end{equation}
where \(\langle \cdot, \cdot \rangle\) is the inner product and \(\mathbf{a}^k\) is a learnable attention vector. Unlike standard attention mechanisms, we omit the softmax normalization to preserve the physical meaning of negative coefficients, which correspond to repulsive interactions. Such negative couplings are common in CDS literature~\cite{righetti2006design,Righetti2008PatternGW,Buono2001ModelsOC,Golubitsky2003symmetry,Pinto2012HexapodR} and are essential for generating diverse patterns.

To stabilize magnitudes and accommodate irregular topologies, we apply symmetric degree normalization akin to GCNs~\cite{Kipf2016SemiSupervisedCW}:
\begin{equation}
    \alpha_{i,j}^k = \frac{e_{i,j}^k}{\sqrt{d_i d_j}},
    \label{eq:sies_att_coeff}
\end{equation}
where \(d_i = |\mathcal{N}(i)|\) and \(d_j = |\mathcal{N}(j)|\). The source and target auxiliary features are constructed as
\begin{equation}
\mathbf{m}_i^\text{s} = \mathbf{q}_i^\text{s} + \Theta_k \mathbf{x}_i, \quad \mathbf{m}_j^\text{t} = \mathbf{q}_j^\text{t} + \Theta_k \mathbf{x}_j,
\label{eq:sies_att_features}
\end{equation}
where \(\mathbf{q}_i^\text{s}\) and \(\mathbf{q}_j^\text{t}\) are projections of the node-wise conditioning input \(\mathbf{Q}\). In synchronization control, they encode desired time-invariant phase relationships; in graph representation learning, they encode initial node features. These constructions yield the two SIES settings described below.

\subsection{SIES for Synchronization Control}

SIES for synchronization control is designed to generate diverse prescribed synchronization patterns. We specify target phase differences \(\mathbf{x}_{\text{dp}} = \{\theta_i\}_{i=1}^N\), where \(\theta_i \in [0, 2\pi]\) is the desired phase of node \(i\) relative to the reference node.

To handle periodicity, we apply circular encoding:
\begin{equation}
\mathbf{q}_i = \bigl[ \sin(\theta_i), \cos(\theta_i) \bigr]^\top \triangleq [q_{i,1}, q_{i,2}]^\top.
\label{eq:dp_encoded}
\end{equation}

The source and target encodings are then obtained via learnable projections:
\begin{equation}
\mathbf{q}_i^\text{s} = \Theta_{\text{dp}}^{\text{s},k} \mathbf{q}_i, \quad \mathbf{q}_j^\text{t} = \Theta_{\text{dp}}^{\text{t},k} \mathbf{q}_j,
\label{eq:cem_source_target}
\end{equation}
where \(\Theta_{\text{dp}}^{\text{s},k}, \Theta_{\text{dp}}^{\text{t},k} \in \mathbb{R}^{F \times 2}\) are independently learned in the default model. This independent source-target parametrization allows the influence of node \(j\) on node \(i\) to differ from the reciprocal influence of node \(i\) on node \(j\). Thus, in the synchronization-control setting, the conditioning matrix in Eq.~\eqref{eq:sies_base} is \(\mathbf{Q}=[\mathbf{q}_1,\ldots,\mathbf{q}_N]^\top\), obtained from \(\mathbf{x}_{\text{dp}}\). Substituting these encodings into the attention mechanism (Eqs.~\eqref{eq:sies_raw_att}--\eqref{eq:sies_att_features}) and coupling equations (Eqs.~\eqref{eq:sies_base} and~\eqref{eq:sies_coupling}, with $\mathcal{C}=\mathcal{C}_{\text{avg}}$) defines the full dynamical system. The model is trained via off-policy reinforcement learning (RL) to learn the coupling terms that drive the CDS toward the target synchronization patterns characterized by the desired phase differences $\mathbf{x}_{\text{dp}}$.

For the non-negative-attention control, we retain the synchronization-control architecture and replace Eq.~\eqref{eq:sies_att_coeff} with the standard softmax coefficient map
\begin{equation}
    \alpha_{i,j}^{k,\mathrm{softmax}} =
    \frac{\exp(e_{i,j}^k)}
    {\sum_{\ell \in \mathcal{N}(i)} \exp(e_{i,\ell}^k)} .
    \label{eq:cem_softmax_control}
\end{equation}
For the source-target-tying control, we retain signed attention but impose
\begin{equation}
    \Theta_{\text{dp}}^{\text{s},k}=\Theta_{\text{dp}}^{\text{t},k}.
    \label{eq:cem_tied_control}
\end{equation}
On a reciprocal graph, this tied constraint makes the task-conditioned part of the raw score exchange-symmetric; with the shared state transformation in Eq.~\eqref{eq:sies_att_features}, it yields \(e_{i,j}^k=e_{j,i}^k\) and hence symmetric degree-normalized coefficients for reciprocal edges. The default independent source and target projections do not impose this restriction.

\subsubsection{Simplified phase model of SIES for synchronization control}\label{sec:simplified_phase_model}
The synchronization-control instantiation of SIES contains multiple learnable components and is not directly amenable to analytical treatment. We therefore consider a reduced phase model under the following assumptions:
\begin{itemize}
	\item Each subsystem is a two-dimensional planar oscillator with a nearly circular limit cycle of radius one and intrinsic angular frequency $\omega_0$.
	\item The oscillators are weakly coupled, such that the coupling does not alter the limit cycle shape.
	\item The interaction graph is fully connected.
\end{itemize}
Under the first assumption, the state along the limit cycle is $\mathbf{x}_0 = [X_0, Y_0]^\top$, satisfying
$$\begin{bmatrix}
	X_0 \\
	Y_0
\end{bmatrix}
=
\begin{bmatrix}
	\cos(\varphi_i) \\
	\sin(\varphi_i)
\end{bmatrix},
\quad
X_0^2 + Y_0^2 = 1,
\label{eq:limit-cycle-states}$$
where $\varphi_i$ denotes the phase variable of the $i$-th node. The phase as a function of the state $\mathbf{x}_i = [X_i, Y_i]^\top$ is
$$\varphi_i(\mathbf{x}_i) = \arctan\left( \frac{Y_i}{X_i} \right).
\label{eq:partial}$$
The partial derivatives are
\begin{equation}
    \begin{aligned}
         \frac{\partial \varphi_i}{\partial X_i} &= -\frac{Y_i}{X_i^2 + Y_i^2},\\
         \frac{\partial \varphi_i}{\partial Y_i} &= \frac{X_i}{X_i^2 + Y_i^2}.
    \end{aligned}
\label{eq:phase_model_partial}
\end{equation}
Under the fully connected assumption, every node has degree $N-1$. Omitting the feature transformation matrices and considering a single attention head yields
\begin{equation}
    \dot{\mathbf{x}}_i = f(\mathbf{x}_i) + \frac{1}{N-1} \sum_{j \neq i} e_{i,j} \mathbf{x}_j,
\label{eq:simp-sies}
\end{equation}
where the summation uses $j \neq i$ due to the complete graph structure considered herein.
Setting $\epsilon = 1/(N-1)$ and treating $\epsilon \sum_{j \neq i} e_{i,j} \mathbf{x}_j$ as a perturbation, the phase model for weakly coupled oscillators~\cite{Kuramoto1984ChemicalOW} (see Section~5.2) yields
\begin{equation}
    \frac{d\varphi_i}{dt} = \omega_0 + \epsilon \sum_{j \neq i} Z(\phi_i) \cdot V_{i,j}(\varphi_j),
\label{eq:phase_raw}
\end{equation}
where
\begin{equation}
    Z(\phi_i) = \bigl( \nabla_{\mathbf{x}_i} \phi_i(\mathbf{x}_i) \bigr) \big|_{\mathbf{x}_i = \mathbf{x}_0(\phi_i)}
= [-\sin(\phi_i), \cos(\phi_i)]
\end{equation}
represents the phase-sensitivity vector with respect to perturbations of the state, and
\begin{equation}
    V_{i,j}(\varphi_j) = e_{i,j}
\begin{bmatrix}
	\cos(\varphi_j) \\
	\sin(\varphi_j)
\end{bmatrix}.
\end{equation}
Substituting these into Equation~\eqref{eq:phase_raw} gives
\begin{equation}
\begin{aligned}
	\frac{d\varphi_i}{dt}
	&= \omega_0 + \frac{1}{N-1} \sum_{j \neq i} e_{i,j} [-\sin(\varphi_i) \cos(\varphi_j) + \cos(\varphi_i) \sin(\varphi_j)] \\
	&= \omega_0 + \frac{1}{N-1} \sum_{j \neq i} e_{i,j} \sin(\varphi_j - \varphi_i).
\end{aligned}
\label{eq:phase_model}
\end{equation}

The reduced phase model permits a limited analytical connection to scale generalization. Theorem~\ref{theo:1} identifies a phase-dependent coupling rule that produces traveling waves for fully connected systems with different \(N\), providing an analytical counterpart to one class of patterns evaluated empirically.

\begin{theorem}
\label{theo:1}
Consider a CDS consisting of $N$ nodes that can be reduced to the following phase model:
\begin{equation}
\dot{\varphi}_i = \omega_0 + \frac{1}{N-1} \sum_{j \neq i} e_{ij} \sin(\varphi_j - \varphi_i), \quad \text{for } i=1,2,\dots,N,
\end{equation}
where $\varphi_i \in [0, 2\pi)$ denotes the phase of the $i$-th node, $\omega_0$ is the intrinsic oscillation frequency of each node, and $e_{ij} \in \mathbb{R}$ represents the coupling strength between the $i$-th and $j$-th nodes ($e_{i,j}\ne e_{j,i},\exists\ i,j$). There exists a coupling function $e_{ij} = A(\theta_i, \theta_j)$ determined solely by the desired phases $\theta_i$ and $\theta_j$ such that, when the desired phases form the traveling-wave configuration $\theta_i = \frac{2\pi (i-1)}{N}$, the same functional form $A(\theta_i, \theta_j)$ generates a traveling-wave solution for fully connected networks with $N > 4$.
\end{theorem}

The proof of this theorem is provided in SI Section \ref{SI_sec: Proof of Theorem 1}.

\subsection{SIES for Graph Representation Learning} \label{sec:sies_gnn_method}

For graph representation learning, the SIES coupling operator (generalized to \(F\)-dimensional states by adjusting the dimensions of the projection matrices accordingly) directly evolves node features on graph datasets. This places it within the family of CDS-inspired GNNs~\cite{Chamberlain2021GRANDGN, Rusch2022GraphCoupledON, Nguyen2023FromCO}. Given a graph \(G\) with \(N\) nodes and initial hidden features \(\mathbf{p}_i \in \mathbb{R}^F\), we initialize the SIES states as \(\mathbf{x}_{i,1}^{(0)} = \mathbf{x}_{i,2}^{(0)} = \mathbf{p}_i\), define \(\mathbf{X}=[\mathbf{x}_{1,1},\ldots,\mathbf{x}_{N,1}]^\top\) and \(\mathbf{Y}=[\mathbf{x}_{1,2},\ldots,\mathbf{x}_{N,2}]^\top\), and evolve them using damped harmonic intrinsic node dynamics defined in Eq.~\eqref{eq:intrinsic_overdamped}.

The node feature evolution follows:
\begin{equation}
\begin{aligned}
\dot{\mathbf{x}}_{i,1} &= \mathbf{x}_{i,2}, \\
\dot{\mathbf{x}}_{i,2} &= -2\zeta\omega_0 \mathbf{x}_{i,2} - \omega_0^2 \mathbf{x}_{i,1} + \left[\mathbf{F}_{\Theta}(\mathbf{Y},\mathbf{Q},G)\right]_i,
\end{aligned}
\label{eq:sies_gnn}
\end{equation}
where \(\mathbf{F}_{\Theta}\) is the same attention-based coupling operator introduced in Eq.~\eqref{eq:sies_base}-\eqref{eq:sies_coupling}, with $\mathcal{C}=\mathcal{C}_{\text{concat}}$. Here, \(\zeta, \omega_0\) denote the damping ratio and natural frequency, respectively.`.

Here, unlike the synchronization-control setting, the node-wise conditioning input is data-driven: we set \(\mathbf{q}_i = \mathbf{x}_{i,2}^{(0)}=\mathbf{p}_i\) and \(\mathbf{Q}=[\mathbf{q}_1,\ldots,\mathbf{q}_N]^\top\). The attention features become
\begin{equation}
\mathbf{m}_i^\text{s} = \Theta_{\text{dp}}^{\text{s},k} \mathbf{x}_{i,2}^{(0)} + \Theta_k \mathbf{x}_{i,2}, \quad \mathbf{m}_j^\text{t} = \Theta_{\text{dp}}^{\text{t},k} \mathbf{x}_{j,2}^{(0)} + \Theta_k \mathbf{x}_{j,2}.
\label{eq:sies_gnn_features}
\end{equation}

To compare signed coupling with established non-negative attention in graph representation learning, we implement GraphCON-GAT as the non-negative counterpart of SIES. The overdamped harmonic dynamics, state initialization, numerical integration, hidden-state construction and evaluation protocol are retained, while the attention-coefficient map is changed from signed degree normalization to the softmax map conventionally used in GAT~\cite{velivckovic2017graph}:
\begin{equation}
\begin{aligned}
\alpha_{i,j}^{k,\mathrm{SIES}} &=
    \frac{e_{i,j}^{k}}{\sqrt{d_i d_j}} \in \mathbb{R},\\
\alpha_{i,j}^{k,\mathrm{GraphCON\text{-}GAT}} &=
    \frac{\exp(e_{i,j}^{k})}{\sum_{\ell \in \mathcal{N}(i)} \exp(e_{i,\ell}^{k})}
    \in [0,1].
\end{aligned}
\label{eq:gnn_sign_control}
\end{equation}
Thus, within this matched CDS-based GNN comparison, performance differences probe the signed-attention coefficient map relative to a standard, previously established non-negative attention counterpart. We interpret this as a controlled comparison of attention operators, with signed interaction capacity being their central functional distinction.

The system is integrated using the symplectic Euler method~\cite{Hairer2010SolvingOD}:
\begin{align}
\mathbf{Y}^\text{next} &= \mathbf{Y} + \Delta t \Bigl(-2\zeta\omega_0 \mathbf{Y} - \omega_0^2 \mathbf{X} + \mathbf{F}_{\Theta}(\mathbf{Y},\mathbf{Q},G)\Bigr), \label{eq:symp_euler_y} \\
\mathbf{X}^\text{next} &= \mathbf{X} + \Delta t \cdot \mathbf{Y}^\text{next}, \label{eq:symp_euler_x}
\end{align}
where \(\mathbf{X}, \mathbf{Y} \in \mathbb{R}^{N \times F}\) collect the feature and momentum states, respectively, and \(\Delta t > 0\) is the time step. This scheme provides an explicit update for the damped second-order evolution.

\subsection{Experimental Details of SIES for Synchronization Control}
All synchronization-control experiments were conducted on a PC equipped with an NVIDIA GeForce RTX 3080 Laptop GPU (8\,GB memory), a 16-vCPU 11th Gen Intel Core i7-11800H CPU @ 2.30\,GHz, 31\,GB of RAM, and running Ubuntu 20.04. Unless otherwise specified, each subsystem follows Hopf intrinsic node dynamics (SI Eq.~\eqref{eq:intrinsic}) with \(\beta=10\), \(\omega_0=2\pi\) and \(\lambda=1\). All evaluations use a single trained SIES model with a simulation time step of $0.01$ s. The model is trained via reinforcement learning on an eight-node network of synaptically coupled oscillators (SI Fig.~\ref{fig:network_topo}b). Complete training details are provided in SI Section~\ref{app:sies_cem_rl}.

\subsubsection{Experimental details of generalization capabilities}\label{sec:method_generalization}
We provide a detailed description of the experimental setup presented in Section~\ref{sec:generalization}. We evaluate three forms of generalization for the trained SIES model: new network scales, unseen phase targets $  \mathbf{x}_{\text{dp}}  $, and intrinsic node dynamics not used during training. For each test case, we apply SIES to fully connected CDSs. The dynamical system is simulated for 500 steps, and the evaluation metrics are computed over the final 300 steps. The plotted metrics are the target-aligned order parameter $  R_{\mathrm{target}}  $ and phase RMSE (SI Eqs.~\eqref{eq:R_target} and~\eqref{eq:phase_rmse}). These metrics first account for the target phase pattern and remove an arbitrary global phase offset, thereby assessing alignment with the desired synchronization pattern rather than conventional in-phase synchrony.

For a traveling-wave target at system scale \(N\), the desired lags are \(\theta_i=2\pi(i-1)/N\). For unseen random targets, each target lag is sampled from the discretized set \(\{2\pi p/N:p=0,\ldots,N-1\}\), with the reference node fixed at zero, and \(2N\) targets are evaluated at each scale. To evaluate intrinsic-dynamics generalization, the trained SIES model (on Hopf intrinsic node dynamics, SI Eq.~\eqref{eq:intrinsic}) is evaluated with Van der Pol dynamics (SI Eq.~\eqref{eq:intrinsic_vdp}) and overdamped harmonic dynamics (SI Eq.~\eqref{eq:intrinsic_overdamped}) without retraining. The latter dynamics converge to a stable equilibrium in isolation and therefore test whether learned coupling can impose collective rhythmic behavior on intrinsically non-oscillatory agents.

\subsubsection{Experimental Details of Convergence Properties and Sparse Interaction}\label{sec:method_convergence}

We present a detailed description of the experimental setup in Section~\ref{sec:convergence}. In this experiment, we evaluate the convergence properties of the trained SIES model in comparison with three baseline models: the fully connected (FC) model~\cite{righetti2006design,Righetti2008PatternGW}, the Salamander model~\cite{ijspeert2007swimming}, and the diffusive model~\cite{Yu2016GaitGW}. The models operate on 4-node CDSs with different coupling topologies. The SIES model and the FC model employ a fully connected topology (SI Fig.~\ref{fig:network_topo}a), the diffusive model adopts a diffusive coupling topology (SI Fig.~\ref{fig:network_topo}d), and the Salamander model employs a nearest-neighbor coupling topology (SI Fig.~\ref{fig:network_topo}c). These topologies follow the default configurations reported in the original literature. The precise definitions and parameter settings of the baseline models appear in SI Section~\ref{sec:comp_sies_cem}.

To evaluate convergence, we record the time required for each model to reach one of three target synchronous modes (trot, walk, and bound). We generate each initial state by independently sampling the phase of every node uniformly randomly on the unit circle and apply an identical set of 1000 such initial states to all models to ensure a fair comparison. The coupling functions that realize the three target synchronous modes for the baseline models are provided in SI Section~\ref{sec:comp_sies_cem}. We evaluate the target-aligned order parameter $  R_{\mathrm{target}}  $ of the simulated CDS states and define the convergence time \(t^*\) as the earliest time at which this order parameter reaches or exceeds 0.999:
\begin{equation}
    t^* = \inf \left\{ t \ge 0 \;\middle|\; R_{\mathrm{target}}(t) \ge 0.999 \right\},
    \label{eq:convergence_time}
\end{equation}
To facilitate visualization, we additionally compute the phase distance metric (PD, defined in SI Equation~\eqref{eq:PD}), which quantifies the separation between each initial state and the corresponding target synchronization mode.

In the sparse-interaction analysis, we generate interaction graphs using the small-world network model~\cite{watts1998collective}. We first construct regular ring lattices by setting the rewiring probability $  p_r = 0  $. For each network scale $  N \in \{20, 40, 60, 80, 100, 120\}  $ and each nearest-neighbor count $  k_n = 2, 4, \dots, N-2  $ (in steps of 2), every node connects symmetrically to its $  k_n/2  $ nearest neighbors on either side of a circulant ring, yielding an undirected $  k_n  $-regular lattice.
For the rewiring experiments, we fix $  k_n = 0.3N  $ (rounded to the nearest even integer) and vary rewiring probability $  p_r  $ from 0 to 0.8. For each $  (N, p_r)  $ combination, we generate 20 independent realizations by starting from the corresponding regular ring lattice and rewiring each edge to a uniformly chosen non-adjacent, non-self node with probability $  p_r  $. Every generated graph is evaluated on $  2N  $ randomly sampled target phase patterns. For each evaluation, the dynamical system is simulated for 500 steps, and the synchronization metrics $  R_{\mathrm{target}}  $ and phase RMSE are computed over the final 300 steps. This simulation and evaluation protocol is consistent with that used in the generalization experiments.

\subsubsection{Experimental details of mechanistic ablation}\label{sec:method-ablation}

The mechanistic ablation in SI Section~\ref{sec:mechanistic_cem} compares full SIES with three matched variants on an 8-node fully connected oscillator system. The variants are organized as a staged control of the SIES coupling mechanism. First, the \textbf{softmax attention} variant retains the projected architecture but replaces signed degree-normalized coefficients with the standard non-negative softmax map in Eq.~\eqref{eq:cem_softmax_control}. Second, the \textbf{signed source-target-tied attention} variant retains signed degree-normalized coefficients but ties source and target task-conditioning projections as in Eq.~\eqref{eq:cem_tied_control}. Third, the \textbf{signed source-target-conditioned state-space aggregation} variant retains signed attention and independent source-target conditioning, but removes high-dimensional state projection and aggregates messages in the original oscillator state space:
\begin{equation}
\dot{\mathbf{x}}_i = f(\mathbf{x}_i) + \sum_{j \in \mathcal{N}(i)} \alpha_{i,j} \mathbf{x}_j\quad \text{for } i=1,2,\dots,N,
\label{eq:s-cpg-att}
\end{equation}
where the coefficients \(\alpha_{i,j}\) are computed by the same signed attention rule with separate source-target conditioning as in full SIES.

All four models use 8 attention heads and feature dimension 64 where applicable. Saved checkpoints from 50,000 to 6,000,000 training steps are evaluated at 50,000-step intervals. Training uses four prescribed 8-node phase configurations, expressed in cycles and multiplied by \(2\pi\) in the oscillator objective:
\[
\begin{aligned}
&[0,0.5,0.25,0.75,0.5,0,0.75,0.25],\\
&[0,0,0.75,0.75,0.5,0.5,0.25,0.25],\\
&[0,0.5,0.5,0,0,0.5,0.5,0],\\
&[0,0,0.5,0.5,0,0,0.5,0.5].
\end{aligned}
\]
Generalization is evaluated on four held-out diagnostic targets:
\[
\begin{aligned}
T_1&=[0,0.125,0.25,0.375,0.5,0.625,0.75,0.875],\\
T_2&=[0,0.25,0.5,0.75,0,0.25,0.5,0.75],\\
T_3&=[0,0,0,0,0.5,0.5,0.5,0.5],\\
T_4&=[0,0.25,0.5,0.75,0.125,0.375,0.625,0.875].
\end{aligned}
\]
For each saved checkpoint, model and target configuration, we simulate identical initial conditions for 500 steps. Target-aligned order parameter \(R_{\mathrm{target}}\) and phase RMSE are computed over the final 300 steps (SI Eqs.~\eqref{eq:R_target} and~\eqref{eq:phase_rmse}).

\subsection{Experimental Details of Robotic Locomotion Using SIES} \label{sec:method_robot_locomotion}

We provide a detailed description of the experimental setup presented in Section~\ref{sec:robot-exp}. For multi-legged locomotion, we apply SIES to fully connected CDS networks to generate rhythmic signals for the robot legs. Each oscillator node controls the rhythm of one leg. The output of node $  i  $ is converted to a normalized progression signal $  p_i = \Phi_{\mathbf{x}_i}/(2\pi)  $, where $  \Phi_{\mathbf{x}_i} \in [0, 2\pi)  $ denotes the phase angle of the state vector $  \mathbf{x}_i  $. This maps the oscillator state to the interval $  [0, 1)  $. Foot trajectories are then generated from the progression signal $  p_i  $, the duty cycle $  \mu \in (0,1)  $, which defines the stance and swing phases, and the turning parameter $  \tau \in [0,1]  $ (SI Section~\ref{app:motion_gen}).
For the fault-tolerant locomotion experiment, SIES runs on a remote PC (Intel i7-11800H, 32 GB RAM, RTX 3080 Laptop GPU) at 50 Hz and transmits phase signals to the robot via ROS over Wi-Fi. This setup is used because the PyTorch Geometric dependencies are not available on the onboard Jetson Orin NX processor. All trajectory planning and low-level motor control are executed locally on the Jetson Orin NX at 50 Hz. The robot receives phase-lag targets and duty cycles according to the number of active legs: a tripod gait with $  \mathbf{x}_{\text{dp}} = [0, \pi, 0, \pi, 0, \pi]  $ and $  \mu = 0.5  $ for six legs; a metachronal gait with $  \mathbf{x}_{\text{dp}} = [0, 2\pi/5, 4\pi/5, 6\pi/5, 8\pi/5]  $ and $  \mu = 0.85  $ for five legs; and a trot gait with $  \mathbf{x}_{\text{dp}} = [0, \pi, \pi, 0]  $ and $  \mu = 0.85  $ for four legs.

\section{Experimental Details of SIES for Graph Representation Learning}
 All experiments were conducted on a computing server equipped with an NVIDIA A800 GPU (80\,GB memory), a 56-vCPU Intel Xeon Gold 6348 CPU @ 2.60\,GHz, 120\,GB of RAM, and running Ubuntu 22.04. Node classification performance is evaluated using accuracy for the multi-class datasets (\textbf{roman-empire}, \textbf{Amazon-ratings}, \textbf{chameleon}, and \textbf{squirrel}) and ROC-AUC for the binary classification datasets (\textbf{Minesweeper} and \textbf{Questions}). Formal definitions of these metrics are provided in the SI Section~\ref{app:eval_gnn}. All data splits and evaluation protocols are consistent with those in the original sources. For each run, the test accuracy (or ROC-AUC score) of the model achieving the highest validation performance is reported. We present mean values together with standard deviations over multiple independent runs that employ different train/validation/test splits, consistent with the experimental protocol in~\cite{Luo2024ClassicGA}.  See complete training details in SI Section \ref{app:sies_gnn}.

\subsection{Datasets} \label{sec:method_gnn_datasets}

We employ six heterophilous graph datasets introduced in~\cite{Platonov2023ACL} to evaluate the node classification performance of SIES for graph representation learning. These datasets serve as standard benchmarks for assessing graph neural networks on node classification tasks in low-homophily regimes. They are widely adopted to examine model generalization when neighboring nodes tend to belong to different classes and encompass diverse real-world and synthetic domains exhibiting varied structural properties (e.g., sparsity, clustering coefficients, and diameter). \textbf{roman-empire} is derived from the English Wikipedia article on the ``Roman Empire''. Nodes represent words in the article, and edges are constructed according to sequential word order or syntactic dependency trees; node labels indicate the syntactic role of each word. \textbf{Amazon-ratings} is constructed from Amazon's co-purchase network, where nodes correspond to products (e.g., books, CDs, DVDs) and edges connect items that are frequently purchased together. The task is to predict the discretized average user rating of each product. \textbf{Minesweeper} is a synthetic dataset simulating the grid structure of the classic Minesweeper game. Nodes represent grid cells, edges connect adjacent cells, and labels indicate whether a cell contains a mine; node features encode local neighborhood information. The \textbf{Questions} dataset originates from the Yandex Q\&A platform and models user-interaction networks. Nodes represent users, edges capture interactions formed through question-answering activities, and the task is formulated as binary classification of user attributes or behaviors. \textbf{chameleon} and \textbf{squirrel} are classic Wikipedia-based heterophilous networks centered on the ``chameleon'' and ``squirrel'' topics, respectively, originally proposed in~\cite{Rozemberczki2019MultiscaleAN}. Nodes represent articles on the given topic, edges denote hyperlinks between articles, and labels correspond to article categories or traffic levels. Because the original data splits of these datasets contain overlapping nodes between the training and test sets, we adopt the filtered versions provided in~\cite{Platonov2023ACL} that remove such overlaps. Detailed statistics and structural characteristics of the datasets are provided in the SI Table~\ref{tab:dataset_stats}.

\subsection{Configurations of the Tested GNN Models} \label{sec:method_gnn_models}

For the enhanced GNN baselines (GCN$^*$, GAT$^*$, and GraphSAGE$^*$), the network architectures and optimal hyperparameter settings for each dataset are consistent with those provided in the original repository of the work~\cite{Luo2024ClassicGA}.

For the CDS-based GNN models, namely GraphCON, KuramotoGNN, and SIES, the model structures are adapted from the heterophilous portion of the GraphCON repository~\cite{Rusch2022GraphCoupledON}. To evaluate the core capabilities of the CDS-based GNNs themselves, these architectures were deliberately constructed without incorporating the additional enhancement techniques employed in the enhanced GNN baselines, such as residual connections, batch normalization, and inter-layer feature dropout.

For the direct attention-form comparison, GraphCON-GAT and SIES share the damped harmonic dynamical backbone, state initialization, integration procedure, hidden dimensionality and evaluation protocol. GraphCON-GAT uses the standard softmax-normalized attention coefficients of GAT, whereas SIES uses the signed degree-normalized coefficient map in Eq.~\eqref{eq:gnn_sign_control}.

Hyperparameter optimization for the CDS-based models across different datasets was performed using Optuna~\cite{Akiba2019OptunaAN}, a Bayesian optimization framework. The search ranges for the basic hyperparameters---namely hidden dimension, number of attention heads, dropout rates, learning rate, and weight decay---are consistent with those defined in~\cite{Luo2024ClassicGA}. Empirical results indicate that edge dropout improves the classification accuracy of CDS-based GNN models. Accordingly, the search range for the edge dropout rate was set to the same range as the feature-dropout rate defined in~\cite{Luo2024ClassicGA}. 

Additionally, the CDS-based GNN models include hyperparameters related to the construction of the underlying dynamical systems and the numerical integrator. The following discrete search ranges were used for these parameters: $K_s \in \{0.5, 1.0, 1.5, \dots, 12.0\}$, $\omega \in \{0.5, 1.0, 1.5, \dots, 5.0\}$, $\zeta \in \{0.5, 1.0, 1.5, \dots, 5.0\}$, $\Delta t \in \{0.1, 0.2, 0.3, 0.4, 0.5\}$, and the number of integration steps in $\{1, 2, 3, \dots, 20\}$.

\section*{Data Availability}
\begin{sloppypar}
The heterophilous graph benchmark datasets used for the graph representation-learning experiments are publicly available through the data resources described in Ref.~\cite{Platonov2023ACL}. 
\end{sloppypar}

\section*{Code Availability} 
Code for model training, evaluation and centipede robot simulation in PyBullet is openly available under the MIT License at \href{https://github.com/JiChern/SwarmBioOscillators/tree/main}{https://github.com/JiChern/SwarmBioOscillators/tree/main}.

\bibliography{sn-bibliography}

\section*{Acknowledgments}
This work was supported by the ``Pioneering and Leading Goose + X'' Science and Technology Program of Zhejiang (Grant No. 2025C01052); the Key Research and Development Project of China National Tobacco Corporation (Grant No. 110202402018); and the National Natural Science Foundation of China (No. 62573382, No. 12171431, No. 62373323, No. 72342025). We thank Chao Zheng for assistance with robot modeling. We also thank Yunhan He for valuable suggestions on the manuscript.

\section*{Author Contributions}
J.Chen performed the theoretical analysis and wrote the original manuscript, assisted by C.Xu and L.Fan. J.Chen and S.Chen implemented the algorithms and conducted simulation experiments. J.Chen and C.Gong conducted physical robot experiments. All authors contributed to discussing results and reviewing the manuscript. C.Xu, L.Fan, and S.Chen supervised the research.

\newpage
\setcounter{figure}{0}
\renewcommand{\figurename}{Extended Data Fig.}
\section*{Extended Data}

\begin{figure}[th!]
\centering
\includegraphics[width=\textwidth]{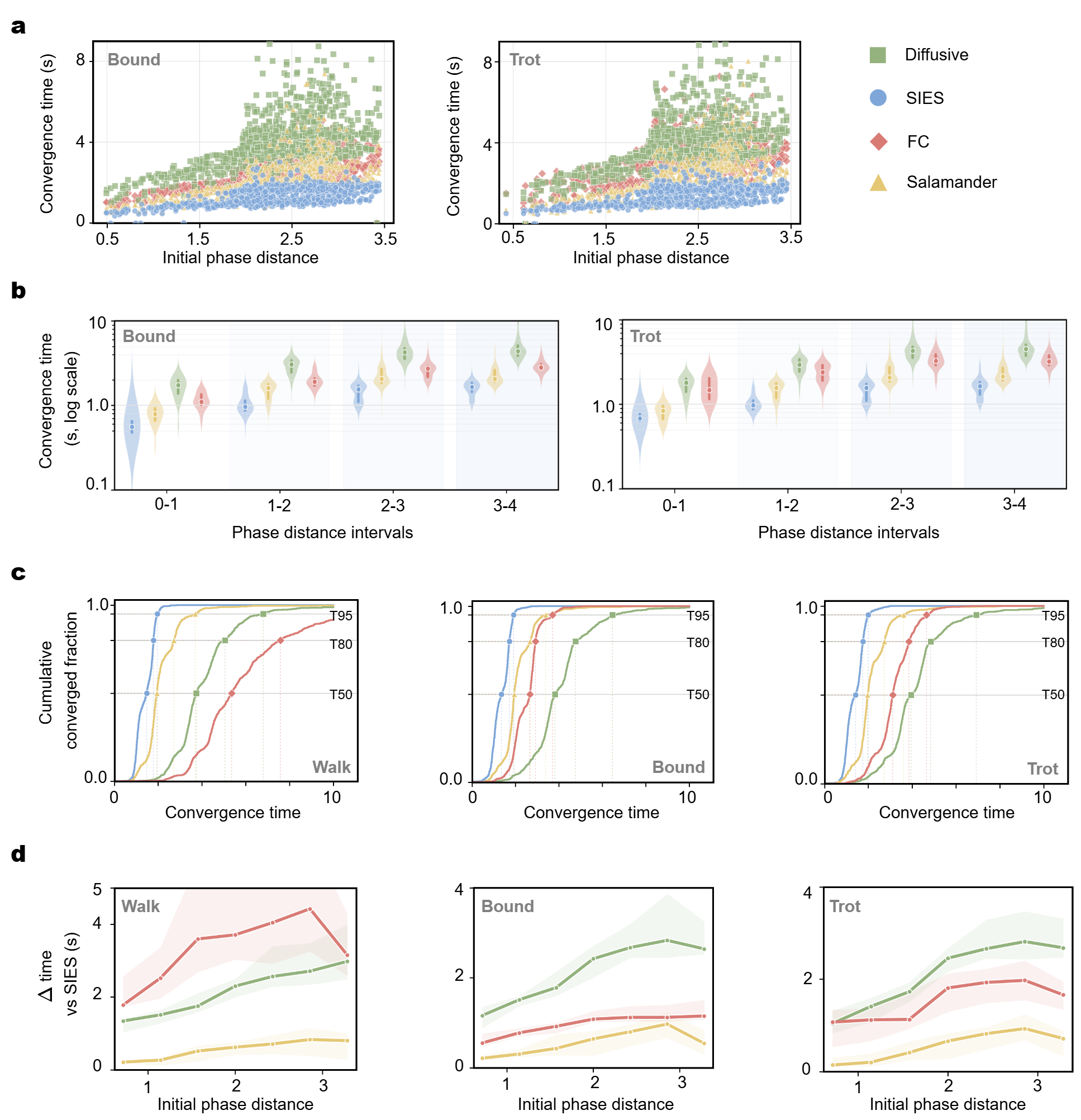}
    \caption{\textbf{Additional results from the convergence evaluation.}
    \textbf{(a)} Trial-level convergence-time scatter plots for the bound and trot targets.
    \textbf{(b)} Convergence-time distributions for the bound and trot targets across phase-distance intervals (log scale).
    \textbf{(c)} Cumulative fraction of converged trials over time for the walk, bound and trot targets, from left to right. Markers indicate the \(50\%\), \(80\%\) and \(95\%\) levels.
    \textbf{(d)} Baseline-minus-SIES convergence-time differences across phase-distance intervals for the walk, bound and trot targets, from left to right.}

\label{fig:basin_att_ext}
\end{figure}

\begin{figure}[th!]
\centering
\includegraphics[width=.9\textwidth]{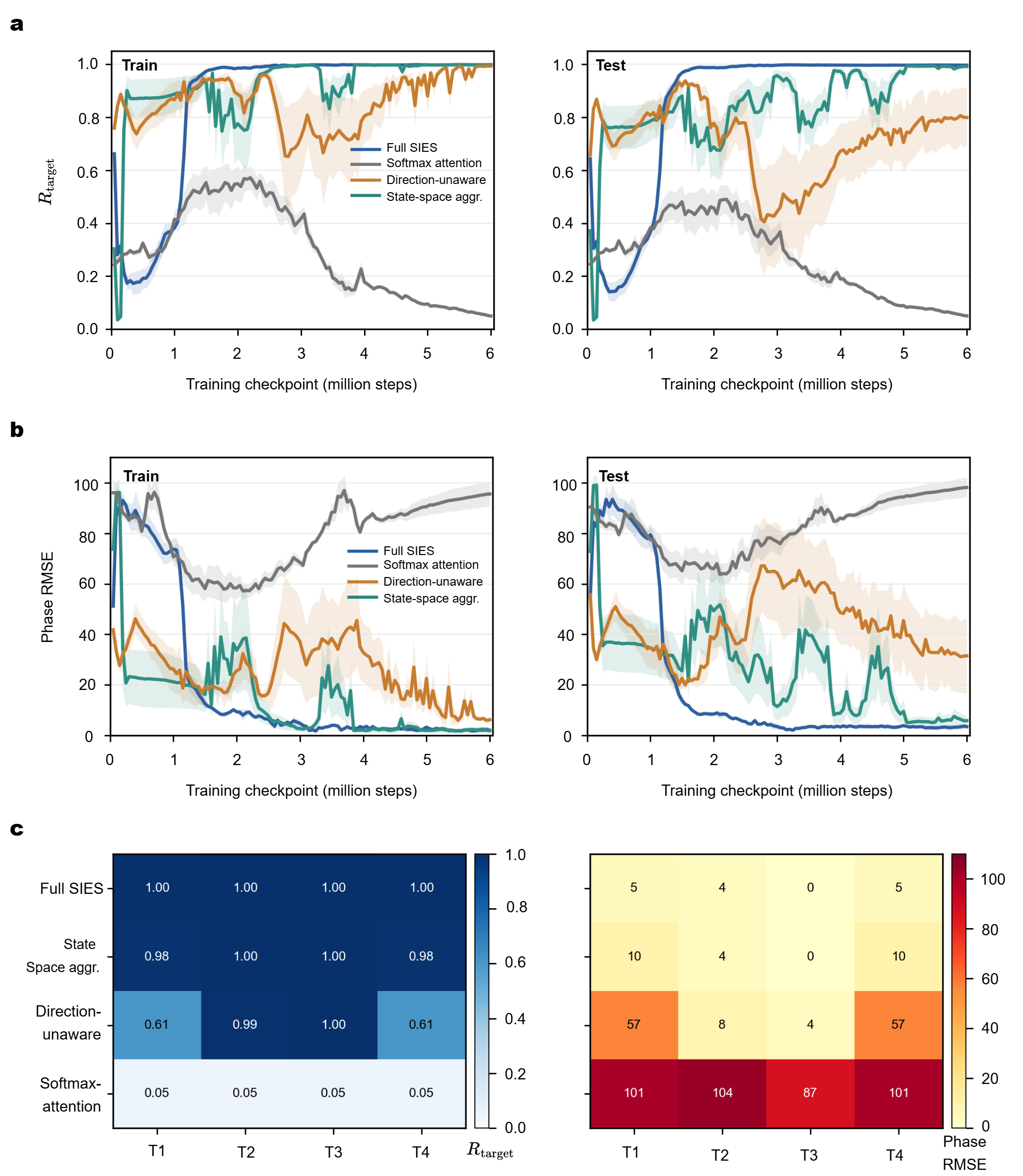}
\caption{\textbf{Mechanistic ablation of SIES for synchronization control across training checkpoints.} \textbf{(a)} Macro-averaged target-aligned order parameter on training and held-out test target sets for full SIES and three matched variants. \textbf{(b)} Corresponding phase RMSE. Shading denotes SEM across the four target modes. \textbf{(c)} Final-checkpoint held-out target aligned order parameter and phase RMSE for four diagnostic targets (T1$\sim$T4).}
\label{fig:ablation_checkpoint}
\end{figure}

\newpage

{
\clearpage

\newpage
\renewcommand{\figurename}{Supplementary Figure}
\renewcommand{\tablename}{Supplementary Table}
\renewcommand\theequation{\Alph{section}.\arabic{equation}}
\setcounter{equation}{0}
\setcounter{figure}{0}
\setcounter{table}{0}
\pagenumbering{arabic}
\renewcommand{\thepage}{S\arabic{page}}
\setcounter{page}{0}

\thispagestyle{empty}
\appendix
\renewcommand{\theHsection}{SI.\Alph{section}}
\renewcommand{\theHsubsection}{SI.\Alph{section}.\arabic{subsection}}
\renewcommand{\theHequation}{SI.\Alph{section}.\arabic{equation}}
\renewcommand{\theHfigure}{SI.\arabic{figure}}
\renewcommand{\theHtable}{SI.\arabic{table}}
\addcontentsline{toc}{section}{Supplementary Information} 
\part{Supplementary Information} 

\begin{Huge}
    Swarm-Inspired Generation of Collective Behaviors in Graph Dynamical Systems
\end{Huge}

\newpage

Here, we provide all notation used in the Supplementary Information in Supplementary Table~\ref{tab: definitions of notations}.


\newcolumntype{D}{>{\centering\arraybackslash} m{3cm} }
\newcolumntype{X}{>{\centering\arraybackslash} m{12cm} }

\begin{table}[htbp]
    \centering
    \renewcommand{\arraystretch}{1.5}
    \begin{tabular}{|D|X|}
    \hline
    $\mathbb{R}$, $\mathbb{R}_{\mathbb{+}}$ & real numbers, positive real numbers\\
    \hline
    $\mathbb{R}^n$,  $\mathbb{R}^{m \times n}$&the $n$-dimensional Euclidean space, the vector space of all $m \times n$ real matrices \\
    \hline
    $O_n$, $I_n$&  the $n$-dimensional zero matrix, the $n$-dimensional identity matrix \\
    \hline
    $||\cdot||$ &the standard Euclidean norm (or induced matrix norm)\\
    \hline
    $\theta_i$ & $\theta_i\in[0,2\pi)$, desired phase lag of the $i$-th node relative to the first node \\
    \hline
    $\phi_i$ &  $\phi_i\in[0,2\pi)$, actual phase lag of the $i$-th node relative to the first node\\
    \hline
    $\varphi_i$ & $\varphi_i\in[0,2\pi)$, actual phase angle of the $i$-th node \\
    \hline
    $\Re(c)$ & real part of a complex number $c$\\
    \hline
    $\Im(c)$ & imaginary part of a complex number $c$\\

     \hline
    \end{tabular}
    \caption{Definitions of notation used in the Supplementary Information}
    \label{tab: definitions of notations}
\end{table}

\newpage

\setcounter{figure}{0}
\renewcommand{\figurename}{Supplementary Fig.}
\section{Additional Figures}

\begin{figure}[th!]
\centering
\includegraphics[width=\textwidth]{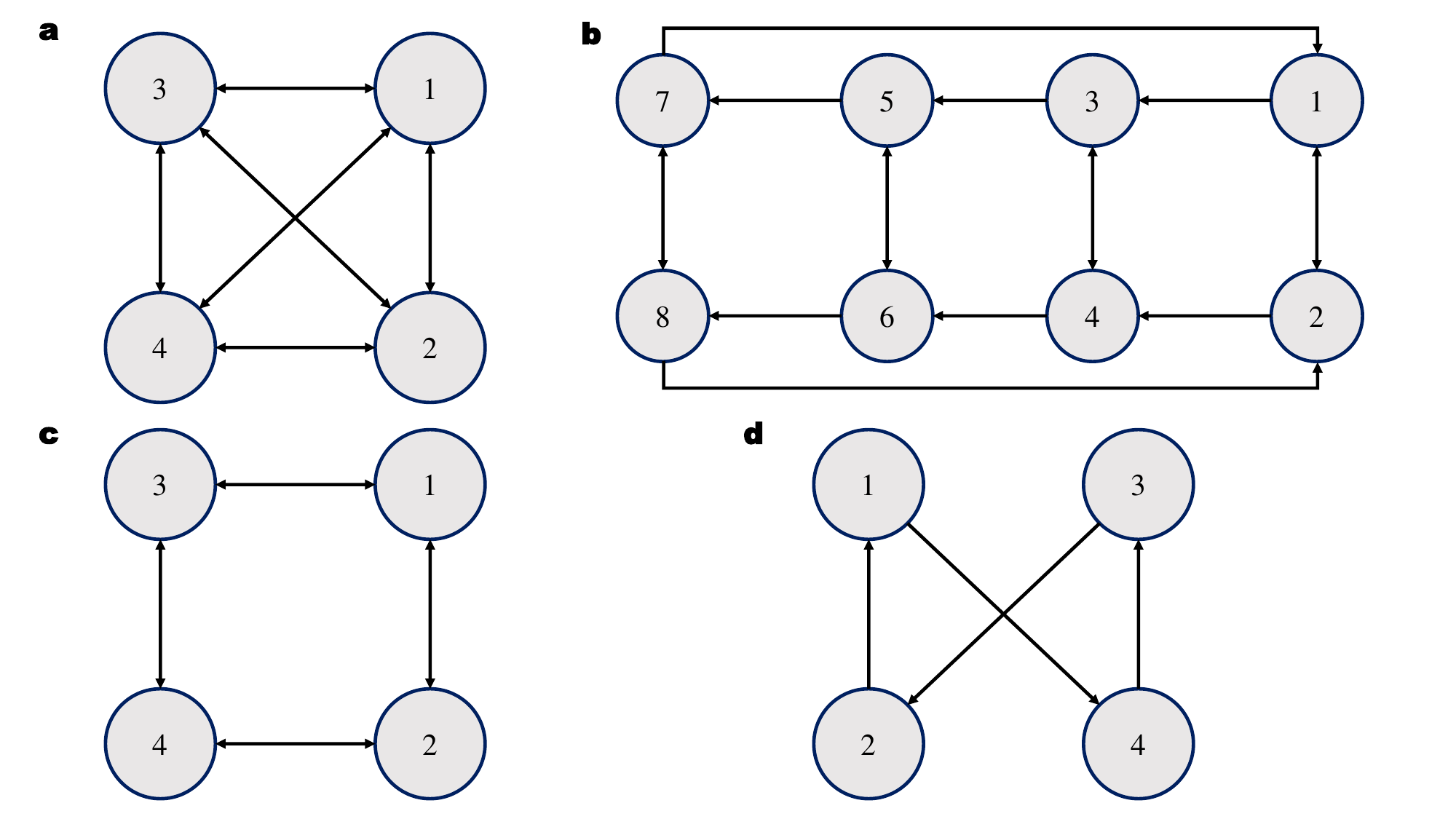}
    \caption{\textbf{Topologies of coupled oscillator networks used for training and testing.} \textbf{(a)} A 4-node fully connected architecture used for constructing the fully coupled baseline (Section~\ref{sec:convergence}) \textbf{(b)} An 8-node synaptically connected architecture used to train the SIES model for synchronization control (Sections~\ref{sec:generalization}--\ref{sec:robot-exp}). \textbf{(c)} A 4-node nearest-neighbor architecture used for the Salamander baseline. (Section~\ref{sec:convergence}) \textbf{(d)} A 4-node diffusively coupled architecture used for the Diffusive baseline. (Section~\ref{sec:convergence})}

\label{fig:network_topo}
\end{figure}

\begin{figure}[th!]
\centering
\includegraphics[width=\textwidth]{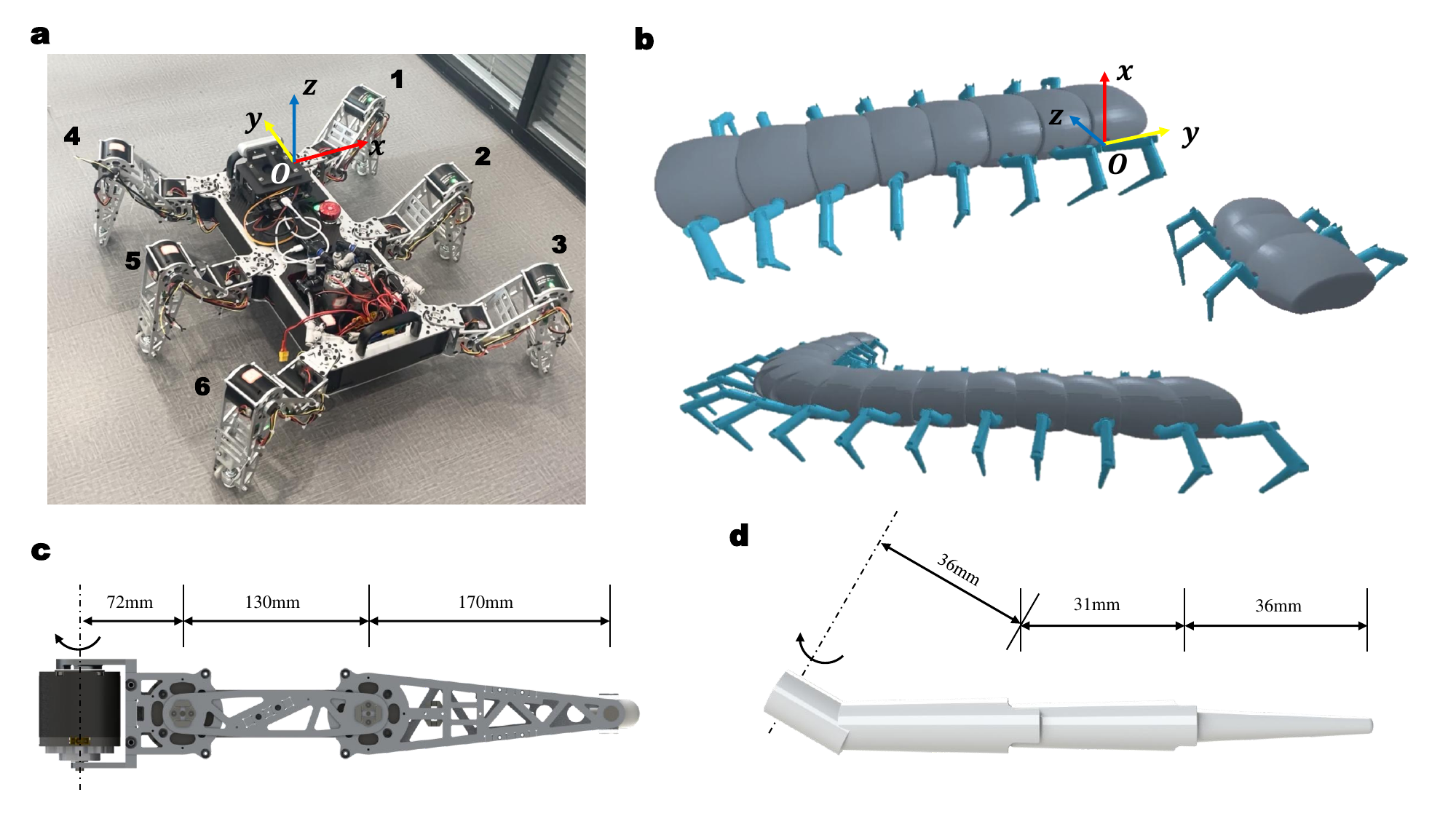}
    \caption{\textbf{An illustration of the physical robot platform and simulated robot model used in this research.} \textbf{(a)} shows the hexapod robot platform. \textbf{(b)} shows the PyBullet virtual simulation model of the centipede robot. \textbf{(c)} illustrates the key dimensions of the hexapod robot leg. \textbf{(d)} illustrates the key dimensions of the centipede robot leg.}
\label{fig:robot_models}
\end{figure}

\section{Intrinsic Node Dynamics}
We evaluate three intrinsic node dynamics in SIES for synchronization control: the Hopf oscillator, the Van der Pol oscillator and the overdamped harmonic oscillator. For a state vector $\mathbf{x}\in\mathbb{R}^2$, the corresponding dynamics $\dot{\mathbf{x}}$ are defined below.
\subsection{Hopf oscillator}
The Hopf oscillator is governed by the following differential equations:
\begin{equation}
\begin{aligned}
    \dot{\mathbf{x}}_{i,1} &= \beta (\lambda - \mathbf{x}_{i,1}^2 - \mathbf{x}_{i,2}^2)\mathbf{x}_{i,1} - \omega_0 \mathbf{x}_{i,2}, \\
    \dot{\mathbf{x}}_{i,2} &= \beta (\lambda - \mathbf{x}_{i,1}^2 - \mathbf{x}_{i,2}^2)\mathbf{x}_{i,2} + \omega_0 \mathbf{x}_{i,1},
\end{aligned}
\label{eq:intrinsic}
\end{equation}
where $\beta$ governs convergence toward the limit cycle, $\lambda$ determines its squared radius, and $\omega_0$ denotes the oscillation frequency. When $\beta>0$ and $\lambda>0$, the system exhibits a stable circular limit cycle of radius $\sqrt{\lambda}$.
\subsection{Van der Pol oscillator}
The Van der Pol oscillator is governed by the following differential equations:
\begin{equation}
\begin{aligned}
    \dot{\mathbf{x}}_{i,1}&=\omega_0\mathbf{x}_{i,2}+\varepsilon\mathbf{x}_{i,1}\left(1-\frac{b\mathbf{x}_{i,1}^2}{3}\right),\\
     \dot{\mathbf{x}}_{i,2}&=-\omega_0\mathbf{x}_{i,1},
\end{aligned}
    \label{eq:intrinsic_vdp}
\end{equation}
where $\omega_0=2\pi$ is defined as in the basic settings, $\varepsilon=1$ represents the scalar coefficient controlling the strength of nonlinear damping, and $b=4$ serves as the amplitude adjustment parameter. 

\subsection{Damped Harmonic Oscillator}

The dynamics of the damped harmonic oscillator are governed by
\begin{equation}
\begin{aligned}
\dot{\mathbf{x}}_{i,1} &= \mathbf{x}_{i,2}, \\
\dot{\mathbf{x}}_{i,2} &= -2\zeta\omega_0 \mathbf{x}_{i,2} - \omega_0^2 \mathbf{x}_{i,1},
\end{aligned}
\label{eq:intrinsic_overdamped}
\end{equation}
where $\zeta$ denotes the damping ratio and $\omega_0$ is the natural frequency. In the intrinsic node dynamics generalization experiments, $\zeta$ is set to 1.5 (exceeding the critical damping threshold of 1) and $\omega_0$ is set to $2\pi$. Under these parameters, the system is overdamped and the associated system matrix possesses two distinct negative real eigenvalues. Consequently, the origin is globally exponentially stable, and each oscillator state converges exponentially to the equilibrium at the origin without exhibiting any oscillatory behavior. This configuration provides a stringent test of whether SIES can impose collective rhythmic behavior on agents whose intrinsic node dynamics are non-oscillatory. The same damped harmonic node dynamics are also used when SIES is applied to graph representation learning.

\section{Proof of Theorem \ref{theo:1}}\label{SI_sec: Proof of Theorem 1}

\subsection{Basic Formulas}

\noindent\textbf{Geometric series sum formula}

A geometric series is a sum of terms where each term after the first is obtained by multiplying the previous term by a constant ratio $ r $. For a series with first term $ a $, common ratio $ r $, and terms from $ k=0 $ to $ N $, the sum is given by:
\begin{equation}
     \sum_{k=0}^{N} a r^k = a \frac{1 - r^{N+1}}{1 - r}, \quad \text{for } r \neq 1. 
     \label{eq:geo_sum}
\end{equation}
\newline

\begin{lemma}
Let $c_0, c_1, \dots, c_{N-1}$ be a sequence of numbers defined as follows: for $m = 1, 2, \dots, N-1$,
\begin{equation*}
    c_m = \frac{1}{N-1} g\left( \frac{2\pi m}{N} \right), 
\end{equation*}
where $g(\delta) = \cos^2 \delta + \sin \delta \cos \delta$, and
\begin{equation*}
    c_0 = -\sum_{m=1}^{N-1} c_m. 
\end{equation*}
Then, for $N > 2$, the following holds:
\begin{equation*}
    c_0 = -\frac{N-2}{2(N-1)}. 
\end{equation*}
\label{lemma:1}
\end{lemma}

\noindent\textbf{Proof.}

\noindent We begin by simplifying the function $g(\delta)$:
\begin{equation}
    g(\delta)=\cos^2 \delta + \sin \delta \cos \delta = \frac{1+\cos2\delta}{2} + \frac{1}{2}\sin2\delta=\frac{1}{2}+\frac{1}{2}\cos2\delta+\frac{1}{2}\sin2\delta.
    \label{eq:g_func}
\end{equation}
Next, we compute the sum $\sum_{m=1}^{N-1} g\left( \frac{2\pi m}{N} \right)$. Using the expression for $g(\delta)$, we have:
\begin{equation}
    \begin{aligned}
            \sum_{m=1}^{N-1} g\left(\frac{2\pi m}{N}\right) &= \sum_{m=1}^{N-1} \left[ \frac{1}{2} + \frac{1}{2} \cos\left(\frac{4\pi m}{N} \right) + \frac{1}{2} \sin\left(\frac{4\pi m}{N} \right) \right]\\
            &=\frac{N-1}{2} + \frac{1}{2}\sum_{m=1}^{N-1}\cos\left(\frac{4\pi m}{N} \right) + \frac{1}{2}\sum_{m=1}^{N-1}\sin\left(\frac{4\pi m}{N}\right).
    \end{aligned}
    \label{eq:lemma1summation}
\end{equation}
To evaluate the sums $\sum_{m=1}^{N-1} \cos\left( \frac{4\pi m}{N} \right)$ and $\sum_{m=1}^{N-1} \sin\left( \frac{4\pi m}{N} \right)$, let $\omega = e^{\mathrm{i} \frac{4\pi}{N}}$, where $\mathrm{i}=\sqrt{-1}$ is the imaginary unit. Then:
\begin{equation}
    \sum_{m=1}^{N-1}\omega^m = \sum_{m=1}^{N-1}e^{i \frac{4\pi m}{N}}=\sum_{m=1}^{N-1}\left[ \cos\left(\frac{4\pi m}{N}\right)+i\sin\left(\frac{4\pi m}{N}\right)\right].
\end{equation}
Here, $\sum_{m=1}^{N-1} \cos\left( \frac{4\pi m}{N} \right)$ is the real part of $\sum_{m=1}^{N-1} \omega^m$, and $\sum_{m=1}^{N-1} \sin\left( \frac{4\pi m}{N} \right)$ is the imaginary part. Using the geometric series formula (Equation~\eqref{eq:geo_sum}), we compute:
\begin{equation}
     \sum_{m=1}^{N}\omega^m=\sum_{m=0}^{N}\omega^m - \omega^0 = \frac{1-\omega^{N+1}}{1-\omega}-1=\omega\frac{1-\omega^{N}}{1-\omega}.
\end{equation}
For $N > 2$, we have $\omega \neq 1$ and $\omega^N = e^{\mathrm{i} 4\pi} = 1$, so:
\begin{equation}
    \sum_{m=1}^{N-1}\omega^m = \sum_{m=1}^{N}\omega^m - \omega^N = 0-  e^{i4\pi} = -1.
\end{equation}
It follows that:
\begin{equation}
\begin{aligned}
    \sum_{m=1}^{N-1} \cos(\frac{4\pi m}{N}) &= \Re(\sum_{m=1}^{N-1}\omega^m) = -1,\\
    \sum_{m=1}^{N-1} \sin(\frac{4\pi m}{N}) &= \Im(\sum_{m=1}^{N-1}\omega^m) = 0.
\end{aligned}
\end{equation}
Substituting these into the expression for the summation (Equation~\eqref{eq:lemma1summation}), we obtain:
\begin{equation}
     \sum_{m=1}^{N-1} g\left(\frac{2\pi m}{N}\right) = \frac{N-1}{2}-\frac{1}{2} = \frac{N-2}{2}.
\end{equation}
Finally, we compute $c_0$:
\begin{align}
    c_0 = -\sum_{m=1}^{N-1}c_m = -\frac{1}{N-1}\sum_{m=1}^{N-1}g\left(\frac{2\pi m}{N}\right)=-\frac{N-2}{2(N-1)}.
\end{align}
This completes the proof. \qed
\newline

\begin{lemma}\label{lemma:2}
    Let $N > 4$ be an integer and $k \in \{0, 1, \ldots, N-1\}$. Then the following three conditions form a complete partition of all values of $k$:
\begin{enumerate}
        \item \( k \not\equiv 2 \pmod{N} \) and \( k \not\equiv N-2 \pmod{N} \),
        \item \( k = 2 \),
        \item \( k = N-2 \).
    \end{enumerate}
\end{lemma}

\noindent\textbf{Proof.}

We first show that these three conditions cover all \( k \in \{0, 1, \ldots, N-1\} \). For any such \( k \), either it satisfies condition (1), or it does not. If not, then by De Morgan's law:
\begin{equation}
    k \equiv 2 \pmod{N} \quad \text{or} \quad k \equiv N-2 \pmod{N}.
    \label{eq:de_morgan}
\end{equation}
Since \( k \) is already in \( \{0, 1, \ldots, N-1\} \), it is its own residue modulo \( N \). Thus, \( k \equiv 2 \pmod{N} \) if and only if \( k = 2 \) (condition (2)), and \( k \equiv N-2 \pmod{N} \) if and only if \( k = N-2 \) (condition (3)). Hence, the three conditions are exhaustive.

Next, we show they are mutually exclusive. Conditions (2) and (3) are clearly disjoint from (1), since (1) explicitly excludes \( k = 2 \) and \( k = N-2 \). It remains to check that (2) and (3) are disjoint: i.e., \( 2 \neq N-2 \). Since \( N > 4 \), we have \( N-2 \geq 3 > 2 \), so indeed \( 2 \neq N-2 \).

Thus, the three conditions form a complete partition of \( \{0, 1, \ldots, N-1\} \). \qed

\subsection{Proof}

To prove Theorem~\ref{theo:1}, we construct a coupling function $A(\theta_i, \theta_j)$ that does not depend explicitly on $N$ and for which the traveling-wave configuration is orbitally stable in fully connected networks with $N>4$. Motivated by the rotational symmetry of this target state, we consider the periodic coupling function
\begin{equation}
	e_{i,j} = A(\theta_i, \theta_j) = \cos(\theta_j - \theta_i)+\sin(\theta_j-\theta_i),
\end{equation}
where $\theta_i=\frac{2\pi (i-1)}{N}, \ i=1,2,\dots, N,$ for desired traveling wave configuration. We demonstrate that when this function is applied to all coupling weights, the phase lags of the emergent waveforms converge to the traveling wave configurations. Under the assumption of $A(\theta_i, \theta_j)$, according to Equation~\eqref{eq:phase_model}, the phase representation of the simplified SIES model is expressed as
\begin{equation}
	\frac{d\varphi_i}{dt} = \omega_0 + \frac{1}{N-1} \sum_{j \ne i} \left[\cos(\theta_j - \theta_i) + \sin(\theta_j-\theta_i)\right] \sin(\varphi_j - \varphi_i).
	\label{eq:phase_model_proof}
\end{equation}

To establish that the actual phase lag $\phi_i = \theta_i \text{,}$ ($i=1,\cdots,N$), we first show that $\varphi_i = \theta_i + \varphi_1 \text{,} \ \forall\ i=1,\cdots,N$ is an equilibrium (phase-locked solution) for this coupled oscillator system. Substituting this condition into Equation~\eqref{eq:phase_model_proof} yields:
\begin{equation}
    \begin{aligned}
	\frac{d\varphi_i}{dt} &= \omega_0 + \frac{1}{N-1} \sum_{j \ne i} \left[\cos(\theta_j - \theta_i) + \sin(\theta_j-\theta_i)\right] \sin(\theta_j+\cancel{\varphi_1} - \theta_i-\cancel{\varphi_1})\\
    &=\omega_0 + \frac{1}{N-1} \sum_{j \ne i} \left[\cos(\theta_j - \theta_i)\sin(\theta_j-\theta_i)+\sin^2(\theta_j-\theta_i)\right ].
    \end{aligned}
\end{equation}
Let $\delta_{ji} = \theta_j - \theta_i = \frac{2\pi (j - i)}{N}$ (modulo $2\pi$). For any fixed $i$, the set of $\{\delta_{ji} \mid j \neq i\}$ is $\left\{ \frac{2\pi k}{N} \mid k = 1, \dots, N-1 \right\}$, which is independent of $i$ due to the cyclic nature of the phases. Then the summation in the equation is identical for every $i$ since it sums over the same set of nonzero multiples of $\frac{2\pi}{N}$. Thus, $\frac{d\varphi_i}{dt}$ equals some constant value (say $\omega_0 + C$) for all $i$.
Since all phases advance at the same angular speed $\omega_0 + C$, the differences $\varphi_j(t) - \varphi_i(t)$ remain constant over time, confirming that this configuration is preserved as a traveling wave. Therefore, $\varphi_i = \theta_i + \varphi_1 \text{,} \ \forall\ i=1,\cdots,N$ is indeed an equilibrium for coupled oscillator system.

To prove the convergence to this equilibrium, we must demonstrate its stability. We analyze the Jacobian matrix evaluated at this equilibrium. It is given as:
\begin{equation}
    J = \begin{bmatrix} -\frac{1}{N-1} \sum_{j\ne 1} e_{1,j}\cos(\theta_j-\theta_1) & \frac{1}{N-1}e_{1,2}\cos(\theta_2-\theta_1) & \dots & \frac{1}{N-1}e_{1,N}\cos(\theta_N-\theta_1)\\
    \frac{1}{N-1}e_{2,1}\cos(\theta_1-\theta_2) & -\frac{1}{N-1} \sum_{j\ne 2}e_{2,j}\cos(\theta_j-\theta_2) & \dots & \frac{1}{N-1}e_{2,N}\cos(\theta_N-\theta_2)\\
    \frac{1}{N-1}e_{3,1}\cos(\theta_1-\theta_3) & \frac{1}{N-1}e_{3,2}\cos(\theta_2-\theta_3) & \dots & \frac{1}{N-1}e_{3,N}\cos(\theta_N-\theta_3) \\
    \vdots & \vdots & & \vdots\\
    \frac{1}{N-1}e_{N,1}\cos(\theta_1-\theta_N) & \frac{1}{N-1}e_{N,2}\cos(\theta_2-\theta_N) & \dots & -\frac{1}{N-1} \sum_{j\ne N}e_{N,j}\cos(\theta_j-\theta_N)
    \end{bmatrix}.
\end{equation}

Now we prove that $J$ is a circulant matrix. The off-diagonal element of the Jacobian matrix is:
\begin{equation}
    J_{i,j} = \frac{1}{N-1}e_{i,j}\cos(\theta_j-\theta_i),\quad i\ne j,
\end{equation}
Let $\Delta = \theta_j-\theta_i=(j-i)\frac{2\pi}{N}$, then
\begin{equation}
    J_{i,j} = \frac{1}{N-1}[\cos^2(\Delta)+\sin(\Delta)\cos(\Delta)] = \frac{1}{N-1} g(\Delta),\quad i\ne j.
    \label{eq:prof1_jij}
\end{equation}
Here, we can further simplify $g(\Delta)$ using double-angle formulas as 
\begin{equation}
    g(\Delta) = \cos^2(\Delta)+\sin(\Delta)\cos(\Delta) = \frac{1}{2} + \frac{\cos(2\Delta)}{2} + \frac{\sin(2\Delta)}{2}.
    \label{eq:prof1_defg}
\end{equation}
As cosine and sine are functions with a period of $2\pi$, then $g(\Delta)=g(\Delta+2k\pi)$ for any $k\in \mathbb{Z}$, which means that the value of $g(\Delta)$ only depends on $(j-i) \mod N$. This shows that the off-diagonal elements of $J$ satisfy the property of circulant matrix.

The diagonal element of the Jacobian matrix is:
\begin{equation}
    J_{i,i} = -\frac{1}{N-1}\sum_{j\ne i}e_{i,j} \cos(\theta_j-\theta_i)=-\sum_{j\ne i} J_{i,j}.
    \label{eq:prof1_jii}
\end{equation}
Thus, $J_{i,i}$ is the summation of all off-diagonal elements of the $i$-th row (multiplied by a constant of $-1$). Due to the circulant property of $J_{i,j}$, the summation is equal to a real number $c_0$ for all $i$. Therefore, $J_{i,i}$ is circulant and $J$ is a circulant matrix.

Now evaluate the eigenvalues of the $J$ matrix. For a circulant matrix $J$, `The eigenvalues are obtained by taking the discrete Fourier transform (DFT) of its first row. Let the elements of the first row of the Jacobian matrix be $ c_0, c_1,\dots,c_m, \dots, c_{N-1} $, based on Equations~\eqref{eq:prof1_jij}, \eqref{eq:prof1_defg} and \eqref{eq:prof1_jii}, the row elements are expressed as:
\begin{equation}
    \begin{aligned}
    c_m &= \frac{1}{N-1}g\left( \frac{2\pi m}{N} \right), \quad m=1,2,\dots, N-1,\\
    c_0 &= -\sum_{m=1}^{N-1} c_m = -\frac{N-2}{2(N-1)}.
    \end{aligned}
    \label{eq:prof1_row}
\end{equation}
Here, the value of $c_0$ is derived according to Lemma~\ref{lemma:1}. The eigenvalues of $J$ are expressed using DFT:
\begin{equation}
    \lambda _{k}=\sum_{m=0}^{N-1}c_m \omega^{km}.
    \label{eq:prof1_lambda}
\end{equation}
where $\omega=e^{\mathrm{i}\frac{2\pi}{N}}$ is a primitive $N$-th root of unity, $\mathrm{i}=\sqrt{-1}$ is the imaginary unit.
\newline

\noindent\textbf{(A) Existence of zero eigenvalue:}
As $J$ is a circulant matrix and the sum of the elements on each row is zero, the all-ones vector $\mathbf{1}=[1,1,\dots,1]^\top$ is an eigenvector, its corresponding eigenvalue is zero. More specifically, we can calculate $\lambda_0$ according to Equation~\eqref{eq:prof1_lambda} as:
\begin{equation}
    \lambda _{0}=\sum_{m=0}^{N-1}c_m = 0.
\end{equation}
\newline

\noindent\textbf{(B) Non-zero $k$ (Showing $\lambda_k \neq 0$ and $\Re(\lambda_k) < 0$):} For $k\ne 0$, according to Equation~\eqref{eq:prof1_lambda} we have
\begin{equation}
    \lambda_k = c_0 + \sum_{m=1}^{N-1} c_m \omega^{km}.
\end{equation}
Substituting Equation~\eqref{eq:prof1_row} and Equation~\eqref{eq:prof1_defg} gives:
\begin{equation}
\begin{aligned}
    \lambda_k &= -\frac{N-2}{2(N-1)} + \sum_{m=1}^{N-1} \frac{1}{N-1}g\left( \frac{2\pi m}{N} \right) \omega^{km}\\
              &= -\frac{N-2}{2(N-1)} + \frac{1}{N-1} \sum_{m=1}^{N-1} \left( \frac{1}{2} + \frac{1}{2} \cos \frac{4\pi m}{N} + \frac{1}{2} \sin \frac{4\pi m}{N} \right) \omega^{k m}\\
              &=-\frac{N-2}{2(N-1)} + \frac{1}{N-1} \left[ \frac{1}{2} S_k + \frac{1}{2} \sum_{m=1}^{N-1} \cos \frac{4\pi m}{N} \cdot \omega^{k m} + \frac{1}{2} \sum_{m=1}^{N-1} \sin \frac{4\pi m}{N} \cdot \omega^{k m} \right],
\end{aligned}
\label{prof1_lambda_long}
\end{equation}
where $S_k = \sum_{m=1}^{N-1} \omega^{k m}$. With Euler's formula for trigonometric functions, we have
\begin{equation}
    \cos \frac{4\pi m}{N} = \frac{\omega^{2m} +\omega^{-2m}}{2}, \quad \sin \frac{4\pi m}{N} = \frac{\omega^{2m} - \omega^{-2m}}{2i}, 
    \label{eq:prof1_euler}
\end{equation}
Substituting this equation into~\eqref{prof1_lambda_long}, we have
\begin{equation}
    \begin{aligned}
    \lambda_k &= -\frac{N-2}{2(N-1)} + \frac{1}{N-1} \left[ \frac{1}{2} S_k + \frac{1}{4} \sum_{m=1}^{N-1} (\omega^{(k+2)m}+\omega^{(k-2)m}) - \frac{\mathrm{i}}{4} \sum_{m=1}^{N-1} (\omega^{(k+2)m}-\omega^{(k-2)m}) \right]\\
              &= -\frac{N-2}{2(N-1)} + \frac{1}{2(N-1)} \left[ S_k + \frac{1}{2} (S_{k+2}+S_{k-2}) - \frac{\mathrm{i}}{2}(S_{k+2}-S_{k-2}) \right] \\
              &= -\frac{N-2}{2(N-1)} +  \frac{1}{2(N-1)}(A+B+C).
    \end{aligned}
    \label{eq:lambdak}
\end{equation}
According to the geometric series sum formula, $\sum_{m=0}^{N-1} \omega^{k m}$ is expressed as
\begin{equation}
    \sum_{m=0}^{N-1} r^m = \frac{1 - r^N}{1 - r}, \quad r = \omega^k.
    \label{eq:prof1_geometric}
\end{equation}
Here, $r^N=\omega^{Nk}=e^{\mathrm{i}\frac{Nk2\pi}{N}}=1$, then Equation~\eqref{eq:prof1_geometric} is expressed as:
\begin{equation}
    \sum_{m=0}^{N-1} r^m = \frac{1 - 1}{1 - r} = 0.
\end{equation}
And this equation can be rewritten as:
\begin{equation}
    1+\sum_{m=1}^{N-1} r^m = 0.
\end{equation}
Therefore $\sum_{m=1}^{N-1} r^m=S_k=A=-1$. 

Now we analyze $B$ and $C$. By our definition,
\begin{equation}
    \begin{aligned}
        S_{k+2} &= \sum_{m=1}^{N-1} \omega^{(k+2) m} = \sum_{m=1}^{N-1} e^{\mathrm{i}\frac{(k+2)m2\pi}{N}},\\
        S_{k-2} &= \sum_{m=1}^{N-1} \omega^{(k-2) m} = \sum_{m=1}^{N-1} e^{\mathrm{i}\frac{(k-2)m2\pi}{N}},\\
        B &= \frac{1}{2} (S_{k+2}+S_{k-2}), \\
        C &= -\frac{i}{2}(S_{k+2}-S_{k-2}),
    \end{aligned}
    \label{eq:BC}
\end{equation}
where $ S_l = \sum_{m=1}^{N-1} \omega^{l m} $ and $ N > 4 $. The term $ S_l $ is the sum of powers of the $ N $-th root of unity:
\begin{equation}
    S_l = \sum_{m=1}^{N-1} \omega^{l m} = \sum_{m=1}^{N-1} e^{\mathrm{i} \frac{2\pi l m}{N} }.
\end{equation}
This sum depends on whether $ l \equiv 0 \pmod{N} $:
\begin{itemize}
    \item if $ l \equiv 0 \pmod{N} $, then $\omega^l = 1$, and:
    \begin{equation}
        S_l = \sum_{m=1}^{N-1} 1 = N-1.\\
        \label{eq:S_l1}
    \end{equation}\\
    \item If $ l \not\equiv 0 \pmod{N} $, then $\omega^l \neq 1$, and according to Equation~\ref{eq:geo_sum}:
    $$\sum_{m=0}^{N-1} \omega^{l m} = 1 + \sum_{m=1}^{N-1} \omega^{l m} = 0,$$
    because the full geometric series sums to zero (since $\omega^l \neq 1$ and $(\omega^l)^N = 1$). Thus:
    \begin{equation}
        S_l = \sum_{m=1}^{N-1} \omega^{l m} = -1.
        \label{eq:S_l2}
    \end{equation}
\end{itemize}

The properties of B and C depend on the value of $k$ (where $k = 0, 1, \dots, N-1$, but we focus on $k \neq 0$ since $\lambda_0 = 0$). Below is a detailed analysis, assuming $N > 4$ to avoid modular degeneracies. Define $S_l=\sum_{m=1}^{N-1}\omega^{lm}$. Using the geometric-series formula, let $r = \omega^l = e^{\mathrm{i}\frac{2\pi l}{N}}$. Then,
\begin{equation}
    S_l = \sum_{m=1}^{N-1} r^m = r \frac{1 - r^{N-1}}{1 - r},
\end{equation}
since $r^N=1$, $r^{N-1}=r^{-1}$, so
\begin{equation}
    S_l = \sum_{m=1}^{N-1} r^m = r \frac{1 - r^{-1}}{1 - r}=-1.
\end{equation}

According to Lemma~\ref{lemma:2}, for $k=\{1,2,\dots,N\}$ and $N>4$, the three conditions (1) $k \not\equiv 2 \pmod{N}$ and $k \not\equiv N-2 \pmod{N}$, (2) $k = 2$ and (3) $k = N-2$ form a complete partition of all values of $k$. Therefore, we analyze $\lambda_k$ by cases based on these three conditions.

\noindent\textbf{General Case: }  $k \not\equiv 2 \pmod{N}$ and $k \not\equiv N-2 \pmod{N}$

In the general case, for $N>4$, neither $ k+2 $ nor $ k-2 $ is congruent to 0 modulo $ N $. Thus, according to Equation~\eqref{eq:S_l2}, $ S_{k+2} = -1 $ (because $ k+2 \not\equiv 0 \pmod{N} $) $S_{k-2} = -1$ (because $ k-2 \not\equiv 0 \pmod{N} $). According to Equation~\eqref{eq:BC}, we compute $B=-1$, $C=0$. Therefore, according to Equation~\eqref{eq:lambdak} the real and the imaginary parts of the eigenvalue $\lambda_k$ are obtained as (since $N > 4 > 1$):
\begin{equation}
    \begin{aligned}
        \Re(\lambda_k) &= -\frac{N-2}{2(N-1)} + \frac{1}{2(N-1)} (-1 - 1) = -\frac{N-2}{2(N-1)} - \frac{1}{N-1} = -\frac{N}{2(N-1)} < 0,\\
        \Im(\lambda_k) &= 0.
    \end{aligned}
\end{equation}
So $\lambda_k$ is real and negative in this general case of $k$.
\newline

\noindent\textbf{Special case:} $k = 2$ (when $k-2 \equiv 0 \pmod{N}$)

For $k=2$, $k+2=4$, $k-2=0$. According to Equation~\eqref{eq:S_l1}-\eqref{eq:S_l2}, $S_{k+2}=-1$, $S_{k-2}=N-1$. Based on Equation~\eqref{eq:BC}, we compute $B=\frac{N-2}{2}$ and $C = \frac{iN}{2}$. Therefore, according to Equation~\eqref{eq:lambdak} the real and the imaginary part of the eigenvalue $\lambda_2$ is obtained as (since $N > 4 > 1$):
\begin{equation}
    \begin{aligned}
        \Re(\lambda_2) &= -\frac{N}{4(N-1)} < 0,\\
        \Im(\lambda_2) &= \frac{1}{2(N-1)} \cdot \frac{N}{2} = \frac{N}{4(N-1)} > 0.
    \end{aligned}
\end{equation}
Thus, $\lambda_2$ is complex with negative real part and non-zero imaginary part.
\newline

\noindent\textbf{Special case:} $k = N-2$ (when $k+2 \equiv 0 \pmod{N}$)

For $k=N-2$, $k+2=N$, $k-2=N-4$. According to Equation~\eqref{eq:S_l1}-\eqref{eq:S_l2}, $S_{k+2}=N-1$, $S_{k-2}=-1$. Based on Equation~\eqref{eq:BC}, we compute $B=\frac{N-2}{2}$ and $C = -\frac{iN}{2}$.  Therefore, according to Equation~\eqref{eq:lambdak} the real and the imaginary part of the eigenvalue $\lambda_{N-2}$ is obtained as (since $N > 4 > 1$):
\begin{equation}
    \begin{aligned}
        \Re(\lambda_{N-2}) &= -\frac{N}{4(N-1)} < 0,\\
        \Im(\lambda_{N-2}) &= -\frac{N}{4(N-1)} < 0.
    \end{aligned}
\end{equation}
Thus, $\lambda_{N-2}$ is complex with negative real part and a non-zero imaginary part.
\newline

In conclusion, for all $k \neq 0$, $\Re(\lambda_k) < 0$. In the general case, $\lambda_k$ is real and negative. In the special cases ($k=2, N-2$, distinct for $N > 4$), $\lambda_k$ has non-zero imaginary part and negative real part. No other $k$ produces $\lambda_k = 0$, as the properties of $B$ and $C$ ensure non-zero values in all cases (no configuration yields $A + B + C$ that cancels the formula to exactly 0 for $k \neq 0$). Thus, there is exactly one eigenvalue equal to 0 ($k=0$), and all others have negative real parts.
\newline

\noindent\textbf{Stability of the equilibrium $\varphi_i=\theta_i+\varphi_1$:}
The single zero eigenvalue corresponds to the eigenvector ($\mathbf{1} = [1, 1, \dots, 1]^\top$), representing a uniform phase shift across all oscillators. Perturbations in this direction neither grow nor decay, reflecting the rotational symmetry of the system, where the equilibrium lies on a one-dimensional manifold of equivalent states. All other eigenvalues have negative real parts, ensuring that perturbations affecting relative phases decay exponentially. Consequently, the equilibrium is locally asymptotically stable in directions transverse to this symmetry. Thus, the system exhibits orbital stability at ($\varphi_i = \theta_i + \varphi_1$), with trajectories near the equilibrium manifold converging to a phase-shifted version of the traveling wave configuration.
\newline

\noindent\textbf{The issue at $N = 4$}

When $N = 4$, $N-2 = 2$, so the two special cases coincide at a single $k = 2$. In this overlapping case:
$k+2 = 4$, $4\equiv 0 \pmod{4}$. According to Equation~\eqref{eq:S_l1}, $S_{k+2} = N-1=3$. Also, $k-2 = 0 \equiv 0 \pmod{4}$. Therefore $S_{k-2} = N-1=3$.
Thus, $B = \frac{1}{2}(3 + 3) = 3$, and $C = -\frac{i}{2}(3 - 3) = 0$.
Plugging into the formula~\eqref{eq:lambdak}:
\begin{equation}
    \lambda_2 = -\frac{4-2}{2(4-1)} + \frac{1}{2(4-1)} (-1 + 3 + 0) = -\frac{2}{6} + \frac{1}{6} (2) = -\frac{1}{3} + \frac{1}{3} = 0.
\end{equation}

This creates a second zero eigenvalue ($\lambda_2 = 0$), so the zero eigenvalue has multiplicity at least two and the target traveling wave is not locally asymptotically stable in all transverse directions.
\newline

\noindent\textbf{Why $N > 4$ resolves this}

For $N > 4$, $2 \not\equiv N-2 \pmod{N}$ (since $N-4 > 0$ and not a multiple of $N$).
The special cases are distinct: $k=2$ and $k=N-2$ are separate, each producing a complex $\lambda_k$ with non-zero imaginary part and negative real part (no additional zeros).
No overlaps or trivial sums occur that force extra zeros, and all non-zero $\lambda_k$ have $\Re(\lambda_k) < 0$. In summary, $N > 4$ prevents the mathematical degeneracy at $k=2 \equiv N-2 \pmod{N}$, ensuring the eigenvalue spectrum behaves as claimed. \qed

\section{Training of SIES for Synchronization Control} \label{app:sies_cem_rl}

In the realm of GNNs, traditional mini-batching approaches falter due to the irregular and interconnected nature of graph data. In our RL approach, leveraging the PyTorch Geometric \cite{fey2019fast} implementation of the advanced mini-batching technique \cite{advminibatch}, we construct a mini-batch from the replay buffer $\mathcal{R}$ by building a ``giant graph'' through the diagonal stacking of adjacency matrices and the concatenation of node features across multiple isolated subgraphs to train the GNN policy. This ensures that message-passing operations remain confined within individual graphs without introducing cross-graph interference or padding overhead. Because convergence of a CDS can depend on initial conditions, each RL training episode starts from independently sampled oscillator states on the unit circle. This exposes the policy to diverse initial phases during optimization.

The training is conducted on an 8-node synaptically connected oscillator system (SI Fig.~\ref{fig:network_topo}b). The SIES model is configured with \( K = 8 \) attention heads and a feature dimension of \( F = 64 \). Node features are processed through a multi-head signed attention mechanism operating in the 64-dimensional feature space. A final dense linear layer with input dimension 64 and output dimension 2 (without non-linearity) was applied to generate the model output. The model was trained to achieve four distinct target phase configurations:
\begin{equation}
    \begin{aligned}
        \mathbf{x}_{\text{dp}}^{(1)} &= [0, \pi, 0, \pi, 0, \pi, 0, \pi], \\
        \mathbf{x}_{\text{dp}}^{(2)} &= [0, \pi, \pi, 0, 0, \pi, \pi, 0], \\
        \mathbf{x}_{\text{dp}}^{(3)} &= [0, 0, \pi, \pi, 0, 0, \pi, \pi], \\
        \mathbf{x}_{\text{dp}}^{(4)} &= [0, \pi, \pi/2, 3\pi/2, \pi, 0, 3\pi/2, \pi/2].
    \end{aligned}
\end{equation}
These targets correspond to different synchronization patterns, including alternating anti-phase, grouped anti-phase, and traveling-wave-like phase relationships. The model is trained on an NVIDIA GeForce RTX 3080 GPU. All experiments for synchronization control were performed using the same trained model without task-specific retraining.

\subsection{Off-policy reinforcement learning framework}
We employ an off-policy reinforcement learning (RL) framework to train SIES. While we do not examine the specific mechanisms by which off-policy RL constructs networks or optimizes parameters, we focus on the general workflow of such methods and their application to our model training.

Each training cycle in off-policy RL consists of two distinct phases: data collection and optimization. These phases operate independently but are connected through a replay buffer $\mathcal{R}$. During the data collection phase, the agent interacts with the environment, and the resulting interaction data are stored in $\mathcal{R}$. In the optimization phase, a mini-batch of data $B$ is randomly sampled from $\mathcal{R}$ to execute the optimization process specified by the RL algorithm, thereby enhancing the performance of the agent.

When implementing off-policy RL for our model, it is necessary to define the environment and agent within our specific context. Our investigation focuses on coupled 2D planar oscillators, where the state of the $i$-th subsystem in Equation~\eqref{eq:sies_base} is defined as $\mathbf{x}_i = [x_{i,1}, x_{i,2}]^\top$. The node connectivity of the coupled system is determined by a graph adjacency matrix $A$. 
\newline

\noindent \textbf{RL agent:} The RL agent comprises two components: a policy, which generates the action $\mathbf{a}$ based on the observation $\mathbf{o}$, and an RL algorithm that optimizes this policy network (including auxiliary networks such as Q-networks). This section provides a detailed description of the first component.

We define the action produced by the policy network as the stacked coupling output of Eq.~\eqref{eq:sies_base}, namely \(\mathbf{a}=[\mathbf{a}_1,\ldots,\mathbf{a}_N]^\top=\mathbf{F}_{\Theta}(\mathbf{X},\mathbf{Q},G)\). In the synchronization-control setting, \(\mathbf{Q}\) is obtained by encoding the desired phase lags \(\mathbf{x}_{\text{dp}}\) using Eq.~\eqref{eq:dp_encoded}. Because the message-passing operator also relies on the graph structure, the graph \(G\), represented computationally by its adjacency matrix \(A\), is supplied as an input to the policy. Thus, the GNN policy is formally expressed as:
\begin{equation}
    \mathbf{a} = \mathbf{F}_{\Theta}(\mathbf{X},\mathbf{Q},G),
    \label{eq:g_cpg_actor}
\end{equation}
where \(\Theta\) denotes all trainable parameters of the coupling operator and policy.
\newline
\newline
\noindent \textbf{RL environment:} At each step of the data collection process, an RL environment generates the observation $\mathbf{o}$ and the reward. The observation is defined as the concatenation of the system states and the encoded desired phases, expressed as
\begin{equation}
\mathbf{o} = \Bigl[ \underbrace{x_{1,1}, x_{1,2}, \dots, x_{N,1}, x_{N,2}}_{\text{system states}}, \underbrace{q_{1,1}, q_{1,2}, \dots, q_{N,1}, q_{N,2}}_{\text{encoded desired phase lags}} \Bigr],
\label{eq:obs}
\end{equation}
where $x_{i,1}$ and $x_{i,2}$ are the two components of the state for the $i$-th node, and $q_{i,1}$ and $q_{i,2}$ are the two encoded desired phase lags of the $i$-th node with respect to the first node, as defined in Equation~\eqref{eq:dp_encoded}.

Once the action for each node is obtained based on the GNN policy, the first derivative of the states is calculated based on Equation~\eqref{eq:sies_base}, and the next state of the system $\mathbf{x}^{\text{next}}_i = [x^{\text{next}}_{i,1}, x^{\text{next}}_{i,2}]^\top$ is computed via numerical integration through time $\Delta t$:
\begin{equation}
\mathbf{x}^{\text{next}}_i = \mathbf{x}_i + \dot{\mathbf{x}}_i \Delta t.
\end{equation}

We assume that the desired phase lags remain constant within an episode; thus, the next observation of the system is given by:
\begin{equation}
\mathbf{o}^{\text{next}} = \Bigl[ x^{\text{next}}_{1,1}, x^{\text{next}}_{1,2}, \dots, x^{\text{next}}_{N,1}, x^{\text{next}}_{N,2}, q_{1,1}, q_{1,2}, \dots, q_{N,1}, q_{N,2} \Bigr].
\label{eq:obs_next}
\end{equation}

After the environment (the coupled oscillator system) transitions based on the action, the reward function $r(\mathbf{o}^{\text{next}})$ is defined to quantify the negative cumulative discrepancy between the desired and actual phase lags in the system:
\begin{equation}
r = -\sum_{i=1}^{N} \sum_{j=1}^{N} \left\| R\left( \theta_i - \theta_j \right) \mathbf{x}_j^{\text{next}} - \mathbf{x}_i^{\text{next}} \right\|,
\label{eq:reward}
\end{equation}
where $R$ denotes a two-dimensional rotation matrix. The term $\theta_i - \theta_j$ represents the desired phase lag between the $i$-th and $j$-th nodes. The expression inside the summation measures the alignment between the actual position $\mathbf{x}_i$ and the position obtained by rotating $\mathbf{x}_j$ according to the desired phase lag. The negative sign makes the reward larger when phase discrepancy is smaller, which is in line with the meaning of the ``reward'' in the context of RL.

For the RL algorithm in the agent, we employ the Twin Delayed Deep Deterministic Policy Gradient (TD3) algorithm~\cite{Dankwa2019TwinDelayedDA} to train SIES. 

Specifically, the critic network consists of a 4-layer multilayer perceptron (MLP) with hidden layers of size 256 and ReLU activations. The actor network is a graph neural network (GNN) policy as defined in Equation~\eqref{eq:g_cpg_actor}. Both the actor and critic networks use a shared learning rate of \(4 \times 10^{-5}\) and are optimized using Adam. During training, a batch of 512 transitions is sampled from the replay buffer; an advanced mini-batching technique from PyTorch Geometric is also applied to make this batch suitable for evaluation in the actor networks. The discount factor is \(\gamma = 0.995\), the target update rate is \(\tau = 0.005\), policy updates are delayed every \(d = 2\) critic updates, exploration noise is drawn from \(\mathcal{N}(0, 0.1)\), and target policy smoothing noise is drawn from \(\mathcal{N}(0, 0.2)\) clipped to \([-0.5, 0.5]\). The training procedure for SIES in synchronization control is shown in Algorithm~\ref{algo:RL}.

\begin{algorithm}
	\textbf{Given:} 
	\begin{itemize}
		\item an off-policy RL algorithm $\mathbb{A}$;
		\item a coupled dynamical system $\mathbb{O}$ of $N$ nodes with its interaction graph $G$ represented by an adjacency matrix $A$;
		\item a GNN policy $\pi_{\Theta}(\mathbf{X},\mathbf{x}_{\text{dp}},G)=\mathbf{F}_{\Theta}(\mathbf{X},\mathbf{Q},G)$ based on Eq.~\eqref{eq:g_cpg_actor};
		\item a set of desired phase difference vectors $\mathcal{D}$.
	\end{itemize}
	\textbf{Initialize:} $\mathbb{A}$, $\mathbb{O}$, GNN policy parameters $\Theta$, replay buffer $\mathcal{R}$ and exploration noise $\epsilon$.

	\ForEach{episode}{
		Sample random initial states $[x_{1,1},x_{1,2},\dots,x_{N,1},x_{N,2}]$ for $\mathbb{O}$, and sample desired phase lags $\mathbf{x}_{\text{dp}}$ from $\mathcal{D}$ \;
		Get encoded desired phase lags $[q_{1,1},q_{1,2},\dots,q_{N,1},q_{N,2}]$ according to Equation \eqref{eq:dp_encoded}\;
		Construct the initial observation of the environment $\mathbf{o}$ according to Equation \eqref{eq:obs}\;
		\ForEach{step of data-collection phase}{
			Obtain system states $\mathbf{X}$ and desired phase lags $\mathbf{x}_{\text{dp}}$ from $\mathbf{o}$\;
			Obtain the interaction graph $G$ from $\mathbb{O}$\;
			Sample an action $\mathbf{a}$ using GNN policy: $\mathbf{a}\leftarrow \pi_{\Theta}(\mathbf{X},\mathbf{x}_{\text{dp}},G) + \epsilon$\;
			Execute the action in $\mathbb{O}$ according to Equation \eqref{eq:sies_base} and observe a new state $\mathbf{o}^{\text{next}}$ according to Equation \eqref{eq:obs_next}\;
			Obtain the reward $r$ according to Equation \eqref{eq:reward}\;
			Store the transition $(\mathbf{o},\mathbf{a}, r,\mathbf{o}^{\text{next}})$ in $\mathcal{R}$\;
			Set $\mathbf{o}\leftarrow\mathbf{o}^{\text{next}}$\;
		
		}
		\ForEach{step of learning phase}{
			Sample a minibatch $B$ from the replay buffer $\mathcal{R}$\;
			Perform one step optimization using $\mathbb{A}$ and minibatch $B$\;
		
	}
	}
	\caption{Training SIES for synchronization control with off-policy reinforcement learning}
	\label{algo:RL}
\end{algorithm}

\section{Training of SIES for graph representation learning} \label{app:sies_gnn}
\subsection{Model Architecture and Forward Pass}

In graph representation learning tasks, the SIES-GNN model adopts an encoder--dynamics--decoder architecture. The encoder is a linear transformation followed by ReLU activation and dropout. It takes raw node feature vectors as input and produces the initial hidden state $\mathbf{Y}_0 \in \mathbb{R}^{N \times F}$, where $N$ is the number of nodes and $F$ is the hidden dimension. SIES then computes adaptive coupling forces between nodes. It employs multi-head attention with distinct linear projections for source and target nodes, allowing the model to generate both positive (attractive) and negative (repulsive) interactions. The resulting coupling signal is projected and injected into the dynamical system. The node representations are evolved for a fixed number of layers using a symplectic Euler integrator (Eq.~\eqref{eq:symp_euler_y} and~\eqref{eq:symp_euler_x}). Finally, a decoder maps the final evolved state to raw class logits. No softmax or sigmoid is applied inside the model; the logits are fed directly to the loss function. This architecture enables the GNN model to capture both attractive and repulsive interactions, which is especially advantageous for node classification on heterophilous graphs.

\subsection{Loss Function}

The supervised loss is computed only on the training nodes. Let $\mathbf{Z} \in \mathbb{R}^{N \times C}$ denote the matrix of raw logits produced by the decoder, where $N$ is the number of nodes and $C$ is the number of classes. Let $\mathcal{T}$ be the index set of training nodes.

For multi-class node classification tasks, we employ the standard cross-entropy loss, defined as
\begin{equation}
\mathcal{L}_{\text{CE}} = -\frac{1}{|\mathcal{T}|}\sum_{i \in \mathcal{T}} \log \left( \frac{\exp(Z_{i, y_i})}{\sum_{k=1}^{C} \exp(Z_{i,k})} \right),
\end{equation}
where $y_i \in \{0,1,\dots,C-1\}$ is the ground-truth class label of node $i$. This loss encourages the model to assign high probability to the correct class while penalizing incorrect ones.

For the binary classification setting (the Questions and Minesweeper datesets), the integer labels are converted to one-hot vectors $\tilde{\mathbf{y}}_i \in \{0,1\}^2$. We then use the binary cross-entropy loss
\begin{equation}
\mathcal{L}_{\text{BCE}} = -\frac{1}{|\mathcal{T}|}\sum_{i \in \mathcal{T}} \sum_{c=1}^{2} \Bigl[ \tilde{y}_{i,c} \log \sigma(Z_{i,c}) + (1-\tilde{y}_{i,c}) \log\bigl(1 - \sigma(Z_{i,c})\bigr) \Bigr],
\end{equation}
where $\sigma(\cdot)$ is the sigmoid function. This formulation applies an independent sigmoid to each logit and measures the binary classification error for both classes.

The two loss functions above are the standard objectives used for the respective classification settings and are computed exclusively on the training nodes.

\section{Datasets}
In the graph representation-learning evaluation of SIES, we use six heterophilous graph datasets introduced in~\cite{Platonov2023ACL}. The statistics and structural characteristics of the datasets are provided in SI Table~\ref{tab:dataset_stats}.

\begin{table}[htbp]
\centering
\caption{Statistics of the heterophilous graph datasets.}
\label{tab:dataset_stats}
\begin{tabular}{lrrrrr}
\toprule
\textbf{Dataset} & \textbf{\# Nodes} & \textbf{\# Edges} & \textbf{\# Features} & \textbf{\# Classes} & \textbf{Homophily Ratio} \\
\midrule
Roman-empire   & 22{,}662 & 32{,}927 & 300   & 18 & 0.05 \\
Amazon-ratings & 24{,}492 & 93{,}050 & 300   & 5  & 0.38 \\
Minesweeper    & 10{,}000 & 39{,}402 & 7     & 2  & 0.68 \\
Questions      & 48{,}921 & 153{,}540& 301   & 2  & 0.84 \\
Squirrel       & 2{,}223  & 46{,}998 & 2{,}089 & 5  & 0.20 \\
Chameleon      & 890      & 8{,}854  & 2{,}325 & 5  & 0.23 \\
\bottomrule
\end{tabular}
\end{table}

\section{Evaluation Metrics} \label{app:metrics}

\subsection{Evaluation Metrics for Synchronization Control} \label{app:eval_cem}
We introduce the actual phase relationship vector $\mathbf{\phi} = [\phi_i]^N$ in the CDS system with $N$ oscillators, where $\phi_1 = 0$ and $\phi_i \in [0, 2\pi)$ represents the actual phase difference of the $i$-th oscillator relative to the first, derived from the output waveform of the CDS system.

\subsection{Target-aligned order parameter and phase RMSE}
The main-text generalization and sparsification figures quantify agreement with a prescribed desired phase pattern, rather than ordinary in-phase synchrony. We extract the instantaneous phase \(\psi_i(s)\) of each waveform through its analytic signal and define the target-corrected phase error
\begin{equation}
    \delta_i^{(\sigma)}(s) =
    \operatorname{wrap}_{[-\pi,\pi)}
    \left(\sigma\psi_i(s)-\theta_i\right), \qquad \sigma\in\{+1,-1\},
\end{equation}
where \(\theta_i\) is the target phase of node \(i\). The orientation \(\sigma\) accounts for the two propagation directions of an oscillatory wave and is selected as the value that maximizes the terminal-window mean order parameter. The target-aligned order parameter is
\begin{equation}
    R_{\mathrm{target}} =
    \frac{1}{L}\sum_{s=S-L+1}^{S}
    \left|
    \frac{1}{N}\sum_{i=1}^{N}
    \exp\left(\mathrm{i}\delta_i^{(\sigma)}(s)\right)
    \right|,
    \label{eq:R_target}
\end{equation}
where \(R_{\mathrm{target}}=1\) indicates exact agreement with the target phase pattern up to a common phase rotation.

To report residual angular error in degrees, we remove the common instantaneous offset
\(\bar{\delta}(s)=\arg\left[N^{-1}\sum_i\exp(\mathrm{i}\delta_i^{(\sigma)}(s))\right]\)
and calculate
\begin{equation}
    \operatorname{Phase\ RMSE} =
    \frac{180}{\pi L}
    \sum_{s=S-L+1}^{S}
    \sqrt{
    \frac{1}{N}\sum_{i=1}^{N}
    \operatorname{wrap}_{[-\pi,\pi)}^2
    \left(\delta_i^{(\sigma)}(s)-\bar{\delta}(s)\right)
    }.
    \label{eq:phase_rmse}
\end{equation}
Low phase RMSE therefore indicates that individual phases match the requested pattern after accounting for arbitrary global rotation.

\subsection{Phase distance}
We employ the phase distance metric to quantify the separation between an initial phase relationship vector \(\mathbf{\phi}\) and the desired phase relationship vector $\mathbf{x}_{\text{dp}} = [\theta_1,\theta_2,\cdots,\theta_N]$ with $\theta_1=0$. This metric is crucial because phases are periodic, making standard Euclidean distance unsuitable; for example, the separation between angles near $0$ and $2\pi$ is small, and direct computation may erroneously interpret small differences across the modular boundary as large ones.

To account for this periodicity, the metric uses a unit circle embedding approach, mapping each phase to two-dimensional Cartesian coordinates. Specifically, the initial phases are transformed into a vector \(\mathbf{v} = [\cos \phi_1, \sin \phi_1, \dots, \cos \phi_N, \sin \phi_N] \in \mathbb{R}^{2N}\), while the desired phases form \(\mathbf{w} = [\cos \theta_1, \sin \theta_1, \dots, \cos \theta_N, \sin \theta_N] \in \mathbb{R}^{2N}\). The phase distance $d_\text{phase}$ is computed as the Euclidean norm in this embedded space:
\begin{equation}
    d_\text{phase} = \|\mathbf{v} - \mathbf{w}\|_2 = \sqrt{\sum_{i=1}^N \left[ (\cos \phi_i - \cos \theta_i)^2 + (\sin \phi_i - \sin \theta_i)^2 \right]} = \sqrt{4 \sum_{i=1}^N \sin^2 \left( \frac{\phi_i - \theta_i}{2} \right)},
    \label{eq:PD}
\end{equation}
where the inherent periodicity of the trigonometric functions ensures that the periodic chordal separation on the unit circle is captured without explicit modular adjustment.

\subsection{Evaluation Metrics for Graph Representation Learning} \label{app:eval_gnn}
In this research, we focus on the node classification ability of SIES for graph representation learning. For a graph dataset with more than two classes, we use the conventional node classification accuracy defined as:

\begin{equation}
    \text{Accuracy} = \frac{1}{|\mathcal{V}_{\rm test}|} \sum_{v \in \mathcal{V}_{\rm test}} \mathbb{I}(\hat{y}_v = y_v)
    \label{eq:accuracy}
\end{equation}
where \(\mathbb{I}(\cdot)\) is the indicator function, \(y_v\) and \(\hat{y}_v\) are the ground-truth and predicted labels of node \(v\), respectively, and \(\mathcal{V}_{\rm test}\) denotes the set of test nodes.

For graph datasets with two classes, we use ROC-AUC~\cite{Fawcett2006AnIT} to deal with class imbalance and to provide a threshold-independent evaluation of the model's discriminative ability. 
ROC-AUC is formally defined as the area under the Receiver Operating Characteristic (ROC) curve, which plots the true positive rate (TPR) against the false positive rate (FPR) at various threshold settings:

\begin{equation}
    \text{ROC-AUC} = \int_{0}^{1} \text{TPR}(t) \, d\text{FPR}(t).
    \label{eq:roc_auc}
\end{equation}
Equivalently, it can be expressed using the Wilcoxon-Mann-Whitney statistic as:

\begin{equation}
    \text{ROC-AUC} = \frac{1}{n_{+} n_{-}} \sum_{i \in \mathcal{P}} \sum_{j \in \mathcal{N}} \mathbb{I}(s_i > s_j)
\end{equation}
where \(n_{+}\) and \(n_{-}\) are the numbers of positive and negative samples, \(\mathcal{P}\) and \(\mathcal{N}\) are the sets of positive and negative instances, and \(s_i, s_j\) are the model's predicted scores (probabilities) for the corresponding samples.

\section{Comparative Models}
\subsection{Comparative Models for Synchronization Control}\label{sec:comp_sies_cem}

\begin{itemize}

\item \textbf{Fully coupled oscillators (FC)~\cite{righetti2006design,Righetti2008PatternGW}}:
For this model, the dynamics of the state $\mathbf{x}_i$ of each node are governed by the following differential equation:
\begin{equation}
    \dot{\mathbf{x}}_{i} = f(\mathbf{x}_i) + \sum_{j} c_{i,j} \mathbf{x}_j,
    \label{eq:fc}
\end{equation}
where $f(\cdot)$ is defined according to SI Eq.~\eqref{eq:intrinsic}, and $c_{i,j}\in\mathbb{R}$ is an element of the coupling matrix $C \in \mathbb{R}^{4 \times 4}$ that incorporates the desired phase lags. The coupling matrices associated with the synchronous models used in this research are as follows:
\begin{equation}
    C_{\text{trot}} = \begin{bmatrix} 0 & -1 & -1 & 1 \\
                                     -1 & 0 & 1 & 1 \\
                                     -1 & 1 & 0 & -1\\
                                     1 & -1 & -1 & 0
                    \end{bmatrix},\quad
                    C_{\text{walk}} = \begin{bmatrix} 0 & -1 & 1 & -1\\
                                                     -1 & 0 & -1& 1 \\
                                                     -1 & 1 & 0 & -1 \\
                                                     1 & -1 & -1 & 0
                                        \end{bmatrix},\quad
                    C_{\text{bound}} = \begin{bmatrix} 0 & -1 & 1& -1 \\
                                                       -1& 0& -1&1\\
                                                        1&-1&0&-1\\
                                                        -1&1&-1&0
                                        \end{bmatrix}.
\end{equation} \\
\item \textbf{Nearest-neighbor coupled oscillators (Salamander)~\cite{ijspeert2007swimming}}:
This model is based on a Kuramoto-like phase model incorporating amplitude dynamics. The dynamics of the phase $\phi_i$ and amplitude $r_i$ of each node are governed by the following differential equations:
\begin{align}
    \dot{\phi}_i &= \omega_0 + \sum_{j} r_j w_{i,j} \sin(\phi_i - \phi_j - \theta_{i,j}), \\
    \ddot{r}_i &= \gamma \left( \frac{\gamma}{4} (\lambda - r_i) - \dot{r}_i \right), \\
    \mathbf{x}_{i,1} &= r_i \cos \phi_i,
    \label{eq:salamander}
\end{align}
where $\omega_0$ denotes the oscillation frequency, $w_{i,j} \in \mathbb{R}^+$ is a constant representing the coupling strength between the $i$-th and $j$-th nodes (determined by nearest-neighbor coupling), $\gamma \in \mathbb{R}^+$ is a constant that controls the convergence speed to the limit cycle, $\lambda$ controls the radius of the limit cycle, and $\theta_{i,j}=\theta_{i}-\theta_{j}$ denotes the desired phase lag between the $i$-th and $j$-th nodes. The desired phase lags $\mathbf{x}_{\text{dp}}$ associated with the gait patterns used in this research are as follows:
\begin{equation}
    \mathbf{x}_{\text{dp,trot}} = [0,\pi,\pi,0], \quad \mathbf{x}_{\text{dp,walk}} = \left[0,\pi,\frac{3\pi}{2},\frac{\pi}{2}\right],\quad \mathbf{x}_{\text{dp,bound}} = \left[0,\pi,0,\pi\right]
    \label{eq:trot_walk}
\end{equation} \\

\item \textbf{Diffusively coupled oscillators (Diffusive)~\cite{Yu2016GaitGW}}:

The network topology of this model is shown in Fig.~\ref{fig:network_topo}d. The dynamics of the state $\mathbf{x}_i$ of each node are governed by the following differential equation:
\begin{equation}
    \dot{\mathbf{x}}_i = f(\mathbf{x}_i) + w (R(\varphi_i) \mathbf{x}_{i+1} - \mathbf{x}_i),
    \label{eq:diffusive}
\end{equation}
where $f(\cdot)$ is defined according to SI Eq.~\eqref{eq:intrinsic}, $w \in \mathbb{R}^+$ denotes the coupling strength, $R \in \text{SO}(2)$ is a 2-dimensional clockwise rotation matrix, and $\varphi_i=\theta_{i}-\theta_{i+1}$ represents the desired phase lag between the $i$-th and $(i+1)$-th nodes. The desired phase lags $\mathbf{x}_{\text{dp}}$ used are defined as in SI Eq.~\eqref{eq:trot_walk}. \\

\item \textbf{Parameter settings}:
For the Diffusive model, the coupling strength $w$ is set to 1. Similar to the aforementioned three models, the Salamander model uses $\lambda=1$ and $\omega_0=2\pi$ to control the radius and oscillation frequency of the limit cycle. The parameter $\gamma$ is set to 20.

\end{itemize}

\subsection{Comparative Models for Graph Representation Learning}

To comprehensively evaluate SIES for graph representation learning, we select two categories of representative comparative models: enhanced classical GNN baselines and GNN models based on coupled dynamical systems (Coupled Dynamical Systems, CDS). The following subsections introduce the specific implementations of these models and discuss their comparative significance with respect to SIES.

\subsubsection{Enhanced GNN baselines}

Recent works have demonstrated that classical Graph Neural Networks (GNNs) can serve as highly strong baselines when not only their hyperparameters are carefully tuned, but also when equipped with a suite of widely adopted engineering tricks. These include Layer Normalization (LayerNorm) or Batch Normalization (BatchNorm), residual connections, dropout applied between GNN layers, regularization techniques, and increased network depth. Following the systematic reassessment in \cite{Luo2024ClassicGA}, which shows that well-optimized classic GNNs can achieve or even surpass the performance of popular graph Transformers on node classification tasks, this work adopts the following enhanced baselines (denoted by $*$):

\begin{itemize}
    \item \textbf{GCN*}: The Graph Convolutional Network from \cite{Kipf2016SemiSupervisedCW}, enhanced with optimized hyperparameters and modern training techniques (LayerNorm/BatchNorm, residuals, and dropout) serving as a strong baseline.
    \item \textbf{GAT*}: The Graph Attention Network from \cite{velivckovic2017graph}, which leverages attention mechanisms to capture inter-node relationships and is further strengthened with the same engineering enhancements.
    \item \textbf{GraphSAGE*}: The GraphSAGE model from \cite{Hamilton2017InductiveRL}, employing sampling and aggregation strategies suitable for large-scale graph data, also augmented with LayerNorm/BatchNorm, residual connections, and interlayer dropout.
\end{itemize}

These enhanced baselines are primarily used to verify whether SIES can surpass the performance limits of traditional GNNs while maintaining architectural simplicity.

\subsubsection{CDS-based GNN models}

The second category consists of GNN architectures inspired by coupled oscillator dynamics, which effectively alleviate the over-smoothing problem commonly encountered in deep GNNs. We compare against the following two representative models:

\begin{itemize}
    \item \textbf{GraphCON}~\cite{Rusch2022GraphCoupledON}: Graph-Coupled Oscillator Networks formulate graph dynamics as a second-order ordinary differential equation system of damped harmonic oscillators coupled through the graph adjacency matrix. We evaluate GraphCON-GAT as the controlled non-negative-attention counterpart of SIES in graph representation learning: it retains the matched oscillator-GNN formulation but computes node attention with the standard softmax map of GAT~\cite{velivckovic2017graph}, restricting coefficients to non-negative values. SIES instead uses signed degree-normalized coefficients (Eq.~\eqref{eq:gnn_sign_control}). This comparison therefore assesses the SIES operator against a literature-grounded non-negative attention mechanism, rather than against a newly constructed unsigned variant.
    
    \item \textbf{KuramotoGNN} \cite{Nguyen2023FromCO}: A Kuramoto model-based continuous-depth GNN that addresses over-smoothing through frequency-synchronization dynamics rather than phase alignment. We include it as a CDS-inspired comparator for node classification.
\end{itemize}

\section{Key Parameters and Dimensions of Physical and Simulated Robots} \label{app:robots}
The key parameters of both the physical robot (SI Fig. \ref{fig:robot_models}a) and the simulation platform (SI Fig. \ref{fig:robot_models}b) are summarized in SI Table~\ref{tab:robot_configs}. The key dimensions of the legs of both robots are shown in SI Figs. \ref{fig:robot_models}c and \ref{fig:robot_models}d.  In simulating the centipede robot, we employ spring-damped joints to model the elastic connections between segments. The damping coefficient for these joints is fixed at 0, while the stiffness coefficient corresponds to the Body Stiffness value listed in SI Table \ref{tab:robot_configs}. For shorter centipede models, we select smaller stiffness coefficients within this range to enhance turning flexibility. For longer models, we choose larger stiffness coefficients to improve stability.

\newcolumntype{E}{>{\centering\arraybackslash} m{4cm} }

\begin{table}[htbp]
    \centering
    \caption{Key parameters of robots}
    \label{tab:robot_configs}
    \renewcommand{\arraystretch}{1.5}
    \begin{tabular}{EEEE}
        \hline
        \textbf{Physical robot} & & \textbf{Simulated robot} &  \\ \hline
        Robot mass & 14 kg & Body Stiffness & $1\times 10^{-4} \sim 5\times 10^{-4}$ \\
        Body mass  & 5.6 kg & Segment mass & 0.05 kg \\
        Leg mass   & 1.4 kg & Leg mass & 0.23 kg \\
        Ipsilateral foot clearance & 220 mm &  Ipsilateral foot clearance & 50 mm\\
        Contralateral foot clearance & 200 mm & Contralateral foot clearance & 60 mm \\
        Body Length & 497 mm & Segment Length & 55 mm \\
        \hline

    \end{tabular}
\end{table}

\section{Ablation Studies of SIES}\label{sec:mechanistic_cem}

A main question is which component of the SIES architecture enables access to phase-lagged collective modes. Improved generalization and convergence could arise from several sources: the signed coefficients in the attention-based coupling term (Methods, Eqs.~\eqref{eq:sies_raw_att}--\eqref{eq:sies_att_coeff}), source-target task conditioning (Methods, Eq.~\eqref{eq:cem_source_target}), or feature-space message construction (Methods, Eqs.~\eqref{eq:sies_coupling} and~\eqref{eq:sies_att_features}). We therefore evaluate the full synchronization-control instantiation of SIES together with matched variants on an 8-node fully connected oscillator system across checkpoints saved from 50,000 to 6,000,000 training steps. We quantify performance by the target-aligned order parameter \(R_{\mathrm{target}}\) and phase RMSE on four training and four held-out diagnostic phase modes. The target vectors, checkpoint protocol, and corresponding controls are defined in Methods, Section~\ref{sec:method-ablation} (Eqs.~\eqref{eq:cem_softmax_control}, \eqref{eq:cem_tied_control} and~\eqref{eq:s-cpg-att}).
\newline

\noindent\textbf{Signed interactions enable phase-lagged mode access.} Replacing the signed, degree-normalized coefficients of SIES (Methods, Eqs.~\eqref{eq:sies_raw_att}--\eqref{eq:sies_att_coeff}) with standard non-negative softmax attention (Methods, Eq.~\eqref{eq:cem_softmax_control}) prevents the tested CDS from forming the diagnostic phase-lagged modes. On the held-out targets, the softmax variant reaches only \(R_{\mathrm{target}}\approx0.050\) with phase RMSE \(\approx 98.25^\circ\) at the final checkpoint (Extended Data Fig.~\ref{fig:ablation_checkpoint}c). Even its best checkpoint yields phase metrics on both training and held-out sets that remain far below those of the full model (Extended Data Fig.~\ref{fig:ablation_checkpoint}a--b). Thus, standard non-negative attention weighting does not explain the observed synchronization-control repertoire; the attention-based coupling term must represent opposing as well as aligning influences.
\newline

\noindent\textbf{Source-target conditioning supports phase-propagating organization.} We next retain signed attention but tie the source and target task-conditioning projections, \(\Theta_{\mathrm{dp}}^{\mathrm{s},k}=\Theta_{\mathrm{dp}}^{\mathrm{t},k}\) (Methods, Eq.~\eqref{eq:cem_tied_control}), instead of using the separate source-target projections of full SIES (Methods, Eq.~\eqref{eq:cem_source_target}). This source-target-tied variant fits the training targets yet fails on held-out traveling-wave-like targets that require asymmetric signed coupling: it reaches a final test performance of \(R_{\mathrm{target}}\approx 0.80\) and phase RMSE \(\approx 31.5^\circ\), with the two most phase-propagating test modes each attaining only \(R_{\mathrm{target}}\approx 0.61\) and RMSE \(\approx 57^\circ\) (Extended Data Fig.~\ref{fig:ablation_checkpoint}c). On the fully connected graphs considered here, tied projections force task-conditioned attention coefficients to remain symmetric, whereas the traveling-wave construction in Theorem~\ref{theo:1} constructs an asymmetric signed coupling matrix. For more irregular target modes the attention coefficient matrix must also be highly asymmetric. Hence, signed coefficients are necessary but not sufficient; separate source-target conditioning is required to access the tested phase-propagating organization.
\newline

\noindent\textbf{Feature-space aggregation improves training trajectory and final fidelity.} A third control preserves signed attention with separate source-target conditioning but aggregates messages in the original oscillator state space rather than in the projected feature space used by full SIES (Methods, Eqs.~\eqref{eq:sies_att_features}, \eqref{eq:sies_coupling} and~\eqref{eq:s-cpg-att}). This variant retains broad mode access, yet it learns more slowly and converges to lower accuracy than full SIES. Using the descriptive checkpoint criterion \(R_{\mathrm{target}}\geq 0.95\) and phase RMSE \(\leq 10^\circ\) on the held-out targets, full SIES first satisfies both thresholds at approximately 1.70 million steps, whereas the state-space aggregation variant does so only at approximately 3.90 million steps (Extended Data Fig.~\ref{fig:ablation_checkpoint}a--b). At the final checkpoint, full SIES achieves \(R_{\mathrm{target}}\approx 1\) and phase RMSE \(\approx 3.5^\circ\), compared with \(R_{\mathrm{target}}\approx 0.99\) and phase RMSE \(\approx 6^\circ\) for state-space aggregation (Extended Data Fig.~\ref{fig:ablation_checkpoint}c).

Together, the staged controls separate an expressive mechanism—signed coupling with separate source-target conditioning—from an architectural mechanism—feature-space aggregation—that accelerates optimization and improves target-pattern fidelity once the mode becomes accessible. This architecture also connects SIES to the lower-level logic of swarms: an agent does not merely average neighbor influence; it modulates whether each neighbor should attract or repel it and whether that influence should flow from source to target or target to source. In SIES, signed and source-target-aware coupling gives this local decision process an explicit dynamical role inside the collective update loop.

\section{Motion Generation for Centipede Simulation and Hexapod Robot Experiment}
\label{app:motion_gen}

To generate coordinated motion (gait) for multi-legged robots, the system converts the output waveform of SIES into a progression signal. The progression $p_i$ is calculated as $p_i = \varphi_{\mathbf{x}_i} / 2\pi$, where $\varphi_{\mathbf{x}_i} \in [0, 2\pi)$ represents the polar angle of the state vector measured from the first coordinate axis. This computation normalizes the phase signal to the range [0, 1). Using the progression signal $p_i$ for each leg, the duty cycle $\mu \in (0, 1)$ (which separates stance and swing phases), and the turning parameter $\tau \in [0, 1]$, the system plans the foot-tip motion in Cartesian space for each leg.

\subsection{Centipede Simulation}

The motion generation system employs the progression signal of the $i$-th oscillator in SIES to control the motion of the $i$-th robot leg. To plan the foot-tip trajectory for the centipede robot, the system sets the mean duty cycle to $\mu_{\text{mean}} = 0.7$, adjusted by $\mu_{\text{div}} = 0.2 \times \tau$ to enable differential steering. The duty cycle for each leg $\mu_i$ depends on the side of the leg:

\begin{equation}
\mu_i =
\begin{cases}
\mu_{\text{mean}} + \mu_{\text{div}}, & \text{for leg } i \text{ on the right side}, \\
\mu_{\text{mean}} - \mu_{\text{div}}, & \text{for leg } i \text{ on the left side}.
\end{cases}
\end{equation}

Given the nominal foot-tip pose of the \( i \)-th leg as \( [x_{\text{nom}}, y_{\text{nom}}, z_{\text{nom}}] \), we plan the foot-tip motion parallel to the \( xOy \)-plane (see SI Fig.~\ref{fig:robot_models}b). During the stance phase (\( p_i < \mu_i \)), the trajectory involves linear backward propulsion along the y-axis with an amplitude \( A = 0.03 \):

\begin{equation}
\begin{aligned}
    x_{\text{des}} &= x_{\text{nom}}, \\
    y_{\text{des}} &= A - 2A \frac{p_i}{\mu_i} + y_{\text{nom}}, \\
    z_{\text{des}} &= z_{\text{nom}}.
\end{aligned}
\end{equation}

During the swing phase (\( p_i \geq \mu_i \)), the leg resets forward with a lift:
\begin{equation}
    \begin{aligned}
    x_{\text{des}} &= x_{\text{nom}} - H \sin\left( \pi \frac{p_i - \mu_i}{1 - \mu_i} \right), \\
    y_{\text{des}} &= -A + 2A \frac{p_i - \mu_i}{1 - \mu_i} + y_{\text{nom}}, \\
    z_{\text{des}} &= z_{\text{nom}},
    \end{aligned}
\end{equation}
where \( H = 0.03 \) denotes the lift height. We apply inverse kinematics to the desired foot-tip positions \( [x_{\text{des}}, y_{\text{des}}, z_{\text{des}}] \) to compute joint angles, which control the corresponding motors of the leg.

\subsection{Hexapod Robot Experiment}

Given the nominal foot-tip pose of the \( i \)-th leg as \( [x_{\text{nom}}, y_{\text{nom}}, z_{\text{nom}}] \), we plan the foot-tip motion parallel to the \( yOz \)-plane (see SI Fig.~\ref{fig:robot_models}a). During the stance phase (\( p_i < \mu_i \)), the trajectory involves linear backward propulsion along the y-axis with an amplitude \( A = 0.12 \):
\begin{equation}
    \begin{aligned}
    x_{\text{des}} &= x_{\text{nom}}, \\
    y_{\text{des}} &= A - 2A \frac{p_i}{\mu_i} + y_{\text{nom}}, \\
    z_{\text{des}} &= z_{\text{nom}}.
    \end{aligned}
\end{equation}
During the swing phase (\( p_i \geq \mu_i \)), the leg resets forward with a lift, and we use three control points to generate a B\'{e}zier curve \cite{bezier1972numerical} for the swing phase trajectory. The control points are:

\begin{equation}
\begin{aligned}
y^{(0)} &= [-A + y_{\text{nom}}, y_{\text{nom}}, A + y_{\text{nom}}], \\
z^{(0)} &= [z_{\text{nom}}, H + z_{\text{nom}}, z_{\text{nom}}],
\end{aligned}
\end{equation}
where \( H = 0.08 \) controls the lift height of the foot-tip. Based on the progression signal \( p_i \) of leg \( i \), the internal progression of the B\'{e}zier curve is \( t = \frac{p_i - \mu_i}{1 - \mu_i} \). The foot-tip position \( (y_{\text{des}}, z_{\text{des}}) \) during the swing phase is computed using the Bernstein polynomial basis \cite{Phillips2003}:

\begin{equation}
\begin{aligned}
y_{\text{des}} &= \sum_{i=0}^{2} \binom{2}{i} (1 - t)^{2 - i} t^{i} y^{(0)}_i, \\
z_{\text{des}} &= \sum_{i=0}^{2} \binom{2}{i} (1 - t)^{2 - i} t^{i} z^{(0)}_i,
\end{aligned}
\end{equation}
where \( \binom{n}{i} = \frac{n!}{i!(n - i)!} \) represents the binomial coefficient. The system sets \( x_{\text{des}} = x_{\text{nom}} \), as the trajectory lies in the \( yOz \)-plane. Using the desired foot-tip positions \( [x_{\text{des}}, y_{\text{des}}, z_{\text{des}}] \), we apply inverse kinematics to compute joint angles, which control the corresponding motors of the leg.

}

\newpage

\backmatter



\end{document}